\documentclass[final,12pt]{colt2021} 

\usepackage{macros}
\def\polyall{\poly(N,K,\frac{1}{1-\gamma},G_{max},\ln(1/\delta),W)}
\def\kn{\1 \{s\in\K^n\}}
\def\unkn{\1 \{s\not\in\K^n\}}
\def\Alg{\textsc{Copoe}}
\def\Algfull{Cautiously Optimistic Policy Optimization and Exploration}
\def\pcpg{\textsc{Pc-Pg}}
\newcommand{\what}[1]{\ensuremath{\widehat{#1}}}
\usepackage{comment}
\usepackage{longtable}
\usepackage{cleveref}
\usepackage{natbib}
\usepackage{enumerate} 

\title[Cautiously Optimistic Policy Optimization and Exploration]{Cautiously Optimistic Policy Optimization  and Exploration \\ with Linear Function Approximation}
\usepackage{times}

\coltauthor{%
 \Name{Andrea Zanette} \Email{zanette@stanford.edu}\\
 \addr{Stanford University}
 \AND
 \Name{Ching-An Cheng} \Email{chinganc@microsoft.com}\\
 \addr{Microsoft Research}%
  \AND
 \Name{Alekh Agarwal} \Email{alekha@microsoft.com}\\
 \addr{Microsoft Research}%
}

\begin{document}

\maketitle

\begin{abstract}
	Policy optimization methods are popular reinforcement learning algorithms, because their incremental and on-policy nature makes them more stable than the value-based counterparts.
	However, the same properties also make them slow to converge and sample inefficient, as the on-policy requirement precludes data reuse and the incremental updates couple large iteration complexity into the sample complexity. These characteristics have been observed in experiments as well as in theory in the recent work of~\citet{agarwal2020pc}, which provides a policy optimization method \pcpg{} that
	can robustly find near optimal polices for approximately linear Markov decision processes but suffers from an extremely poor sample complexity compared with value-based techniques.

In this paper, we propose a new algorithm, \Alg, that overcomes the sample complexity issue of \pcpg{} while retaining its robustness to model misspecification.
Compared with \pcpg,
\Alg{} makes several important algorithmic enhancements, such as enabling data reuse, and uses more refined analysis techniques, which we expect to be more broadly applicable to designing new reinforcement learning algorithms.
The result is an improvement in sample complexity from $\widetilde{O}(1/\epsilon^{11})$ for \pcpg{} to $\widetilde{O}(1/\epsilon^3)$ for \Alg, nearly bridging the gap with value-based techniques.
\end{abstract}


\section{Introduction}
\label{sec:intro}
In building real-world learning systems, it is desirable to have algorithms that possess strong sample complexity guarantees under favorable model assumptions, while being robust to model misspecification. This need of robust adaptivity is particularly crucial to reinforcement learning (RL) applications, where we do not have the luxury of tuning our modeling choices through repeated experimentation with a static dataset.

Nonetheless, the intertwined complexity of credit assignment and exploration inherent to RL makes designing such an algorithm challenging.
Most provably efficient RL algorithms with function approximations~\citep{yang2020reinforcement,jin2020provably,zanette2020provably,ayoub2020model,zhou2020nearly,jiang17contextual} require certain structural assumptions on the environment's regularity in order to provide sample complexity guarantees.
These conditions are in some sense necessary, especially for high dimensional problems; otherwise, the learner in the worst case would require exponentially many samples before discovering any useful information (see e.g.~\citep{kakade2003sample,krishnamurthy2016pac,weisz2020exponential}).
However, these provably efficient RL algorithms are typically not robust to model misspecification, because their performance guarantees allow for only small $\ell_\infty$-bounded perturbations from their assumptions.
Can we design RL algorithms that offer strong performance guarantees when the model assumption holds and degrade more gracefully with model misspecification, such as according to \emph{average} errors?

In this paper, we study this question in the context of policy optimization methods in the exploration setting of an approximately linearly parametrizable Markov decision process (MDP) model, which includes approximately linear or low-rank MDPs studied for example by \citet{yang2020reinforcement,jin2020provably,zanette2020frequentist}.
Policy optimization methods are some of the most classical~\citep{williams1992simple,sutton1999policy,konda2000actor,kakade2001natural} as well as widely used approaches for RL~\citep{schulman2015trust,schulman2017proximal}. Their practical success is largely due to the flexibility to work with differentiable policy parameterization and the capability of directly optimizing the objective of interest.
The latter aspect, in particular, has been theoretically shown to make these techniques robust to model misspecification to a much greater degree than the value- or model-based counterparts~\citep{agarwal2020optimality}. However, for the exploration setting which we study here, relatively few results exist for provably efficient policy optimization methods~\citep{agarwal2020pc,shani2020optimistic,cai2020provably}; we include additional related work in \cref{sec:Literature}.

The closest work motivated by similar reasons to ours is the recent paper of~\citet{agarwal2020pc}, which proposes an algorithm called \pcpg{} that optimizes policies by performing natural gradient ascent \citep{kakade2001natural} to solve a sequence of optimistic MDPs. The guarantees of \pcpg{} exhibits the sort of robustness to misspecification that we desire: the algorithm succeeds whenever the state-action features can linearly approximate the state-action value functions of the learner's policies, with an approximation error defined in an average sense under the visitation distribution of a fixed comparator policy of interest (a notion called the \emph{transfer error} in \citep{agarwal2020optimality}).
As shown in \citet{agarwal2020pc}, this type of error dependency allows for nicer guarantees in misspecification settings such as approximate state aggregations and in individual examples where value- or model-based techniques fail.

However, the robustness of \pcpg{} comes at a steep price: to learn an $\epsilon$-suboptimal policy \pcpg{} requires $\tilde{\Omega}(1/\epsilon^{11})$ number of samples!
This sample inefficiency leads us to ask whether such a trade-off is necessary for a nicer notion of model misspecification.

In this work, we present a new algorithm \Alg{}  (\emph{\Algfull{}}), which builds on \pcpg{} but improves its sample complexity in three crucial ways:
\begin{itemize}

\item \textbf{Pessimistic evaluation with optimistic bonus:} Like most exploration methods, we use the idea of reward bonuses to realize optimism in the face of uncertainty. However, in the optimistic MDP with bonus, we perform \emph{pessimistic} value function estimation for our policies (hence \Alg{} is cautiously optimistic). This trick leads to one-sided errors in our value estimates, which in turn yields important savings in sample complexity.
\item \textbf{Adaptive schedule for MDP update:}
We devise an adaptive scheme to construct the optimistic MDPs in order to avoid repeatedly solving similar optimistic MDPs. While \pcpg{} collects a fixed number of samples with the solution it finds for each optimistic MDP, we use a variable number of samples based on a data-dependent quantity and a doubling schedule. This effectively replaces $O(N)$ rounds of data collection in \pcpg{} with ${O}(d\log N)$ rounds in \Alg{} when we perform $N$ iterations with a $d$-dimensional feature map. This is the \emph{primary source of our sample complexity improvements}, and is enabled by a new concentration inequality for inverse covariance matrices.

\item \textbf{Data reuse via importance sampling:} We show that due to the relative stability of natural gradient ascent in policy optimization, data collected by one policy can be reused to perform many policy updates with basic importance weighted Monte Carlo return estimates, without incurring excessive estimation variance as commonly conjectured. This observation allows us to avoid collecting fresh samples for every policy update as in \pcpg, while keeping its robustness property originating from using Monte Carlo estimation.
\end{itemize}

These algorithmic innovations, along with improvements in the analysis, yield the following informal result for linear MDPs.
Please see \cref{thm:main:MainResult} for the general results.
\begin{theorem}[Informal result for linear MDPs] For a linear MDP~\citep{jin2020provably} with a $d$-dimensional feature map, \Alg{} finds an $\epsilon$-optimal policy with probability at least $1-\delta$ using at most $\widetilde {O}\Big(\frac{d^3\log(1/\delta)}{(1-\gamma)^{13}\epsilon^3}\Big)$ samples  from the MDP.
 \end{theorem}
\Alg{} further retains  the same dependence on transfer error as \pcpg{} when the linear MDP assumption is violated, thereby yielding an improved sample in complexity \emph{without any sacrifice of the robustness to model misspecification}. In addition to the aforementioned algorithmic improvements, our analysis leverages a new covariance matrix concentration result (\cref{lem:Covariance}), which might be of independent interest.

While our algorithm is motivated by approximately linear MDPs, our new result begets the question of whether our algorithm and analysis can be further improved to match the $\approx \frac{d^2}{\epsilon^2}$ sample complexity of the best value-based methods for linear MDPs \citep{zanette2020learning} (notice that our $\widetilde O$ notation hides a dependence on $\log|\ActionSpace|$). We believe that this requires an even stronger data reuse as the variance of importance sampling limits how far back we can go in terms of reusing data from past policies. Estimators based on Bellman backups, such as Fitted Q-iteration and Least Square Policy Evaluation~\citep{bertsekas1995dynamic,sutton2018reinforcement}, can perform a more effective data reuse, but it is unclear if they exhibit a similar robustness to model misspecification. Further investigating these questions is a promising future direction.

\section{Preliminaries}
We consider a discounted infinite-horizon MDP~\citep{puterman1994markov} $M = (\mathcal{S}, \mathcal{A},p, r, \gamma)$ defined by a possibly infinite state space $\mathcal{S}$, a finite  action space $\mathcal{A}$, a discount factor $\gamma \in [0,1)$, and for every state-action pair $(s,a)$, a reward function $r(s,a)$ and a transition kernel $p(\cdot \mid s,a)$ over the next state.
A stationary, stochastic policy $\pi$ maps a state $s\in \StateSpace$ to a probability function $\pi(\cdot \mid s)$ over the actions in $\mathcal{A}$. A policy $\pi$ then induces a distribution over states and actions $d^\pi_{s_0}(s,a) = (1-\gamma) \sum_{t=0}^\infty \gamma^t\textstyle{\Pro^\pi}(s_t = s, a_t = a|s_0),$
 which is the normalized discounted sum of probabilities that the state action $(s_t,a_t)$ at time step $t$ equals $(s,a)$ under the probability function $\Pro^\pi$ associated to the Markov chain induced by $\pi$, with the start state being $s_0$.
Sometimes we also condition on an initial state and an initial action, and we omit any conditioning when it is clear from the context; when the conditioning on the first state and action $(s,a)$ is made explicit, we write $\E_{(s',a') \sim \pi \mid (s,a)}$.  A policy $\pi$ also defines a state-action value function $Q^{\pi}$ and a state value function $V^\pi$, which are
$$Q^\pi(s,a) \defeq \sum_{t=0}^\infty \gamma^t\E_{(s',a') \sim \pi \mid (s,a)}r(s',a') \quad\mbox{and}\quad V^{\pi}(s) \defeq \E_{a \sim \pi(\cdot \mid s)}Q^{\pi}(s,a).$$
For a function $Q~:~\mathcal{S}\times\mathcal{A}\to \mathbb{R}$, we also overload the notation $Q(s,\pi)$ to denote $\E_{a\sim \pi} Q(s,a)$ (e.g. we can write $V^\pi(s)= Q^\pi(s,\pi)$). The corresponding advantage function for $\pi$ is defined as $A^{\pi}(s,a) \defeq Q^\pi(s,a) - V^\pi(s)$. Under some regularity assumptions there exists an optimal stationary policy $\pistar$ whose state and state-action value functions are $\Vstar(s) =  \sup_{\pi} V^{\pi}(s)$ and $\Qstar(s,a) = \sup_{\pi} Q^{\pi}(s,a)$. We also write $Q^\pi(s,a;r)$ or $V^\pi(s;r)$ when emphasizing that the reward function $r$ defines these values.

In this paper, we study linear function approximation under an approximate version of the linear MDP model below; the exact approximation notion is given in \fullref{def:TransferError}.
\begin{definition}[Linear MDP~\citep{jin2020provably}] An MDP $\M = (\mathcal{S}, \mathcal{A},p, r, \gamma)$ is \emph{linear} if there exists a known mapping $\phi~:~\mathcal{S}\times\mathcal{A}\to \mathbb{R}^d$ and a positive measure $\mu~:~\mathcal{S}\to \mathbb{R}^d$ such that for all $s, a, s'$, we have $p(s'|s,a) = \phi(s,a)^\top\mu(s')$.
\label{def:linearMDP}
\end{definition}
Linear MDPs have the attractive property that for any function $f:\mathcal{S}\to \mathbb{R}$, there is  $w_f \in \mathbb{R}^d$ such that $\E_{s'\sim p(\cdot|s,a)}f(s') = w_f^\top\phi(s,a)$. We make the normalization assumption that $\|\phi(s,a)\|_2 \leq 1$ and for any function $f:\mathcal{S}\to \mathbb{R}$ such that $\|f\|_\infty \leq \tfrac{1}{(1-\gamma)^2}$, we have $\|w_f\|_2 \leq W = \widetilde O(\frac{1}{(1-\gamma)^2})$.
The $O(\cdot)$ notation hides constant values and the $\widetilde O(\cdot)$ notation hides constants and $\polylog(d,\frac{1}{1-\gamma},\frac{1}{\epsilon},\frac{1}{\delta})$, where $\delta$ is the failure probability and $\epsilon$ is the suboptimality.
For a symmetric positive definite matrix $\Sigma$ and a vector $x$, we define $\| x \|_{\Sigma} = \sqrt{x^\top \Sigma x}$.

\begin{algorithm}[!b]
	\caption{\Alg: Cautiously Optimistic Policy Optimization and Exploration}
	\label{alg:driver}
	\begin{algorithmic}[1]
	\STATE \textbf{Parameters}: $N, \lambda, \beta$
	\STATE Initialize $\widehat \Sigma^1 = \lambda I, \pi^0(\cdot \mid \cdot) = \text{Unif}(\mathcal A), \underline n =1$ 
	\FOR{$n = 1,2,\dots,N$}
	\STATE  Update policy cover $\pi_{cov}^n = \pi^{0:n-1}$
	\IF{$ \det(\widehat \Sigma^n) > 2\det(\widehat \Sigma^{\underline n}) $ or $n = 1$} \label[line]{line:lazyupdate}
	\STATE Update known set $\K^n$ in~\eqref{eqn:known-set} and bonus $b^n$ in~\eqref{eqn:bonus}; Set $\underline n \gets n$
	\STATE $\pi^{n} \leftarrow \textsc{Solver}(\pi_{cov}^n,b^{n},\K^n)$\label{algline:call_solver}
	\ELSE
	\STATE $\pi^n \leftarrow  \pi^{\underline n}$, $\K^n \leftarrow \K^{\underline{n}}$, $b^n \leftarrow b^{\underline{n}}$
	\ENDIF
	\STATE Sample $\phi^n \leftarrow \textsc{FeatureSampler}(\pi^{\underline n})$ and update $\widehat \Sigma^{n+1} = \widehat \Sigma^{n} + (\phi^n)(\phi^n)^\top$ \label[line]{line:Sigmahat}
	\ENDFOR
	\end{algorithmic}
	\end{algorithm}

\section{Algorithm}
We present the algorithm \emph{Cautiously Optimistic Policy Optimization and Exploration} (\Alg{} ), which is summarized in \cref{alg:driver}. \Alg{}  builds on the \pcpg{} algorithm of~\citet{agarwal2020pc}, with important improvements in design to obtain a better sample complexity. Like \pcpg, \Alg{} is a two-loop algorithm, where the outer loop sets up a sequence of optimistic policy optimization problems which are then solved in the inner loop.

In the $n$th outer loop, we define the \emph{policy cover}, $\pi_{cov}^n$, as the mixture of all the policies discovered so far, and update its empirical cumulative covariance matrix $\widehat{\Sigma}^n$.
We use $\widehat{\Sigma}^n$ to estimate the state-actions that the current policy cover $\pi_{cov}^n$ can confidently explore.

If the current policy cover $\pi_{cov}^n$ can explore a sufficiently larger space than the old policy cover can (which is measured as the change of the covariance matrices in line~\ref{line:lazyupdate}), we proceed to update the learner's policy $\pi^n$. 
To this end, we first define the \emph{known state-actions}, $\K^n$, based on $\widehat{\Sigma}^n$, which can be thought of as the subset of state-actions that can be reached with enough probability under $\pi_{cov}^n$. Using $\K^n$, we create the optimistic MDP for the inner policy optimization by augmenting the original MDP $\M$ with a reward bonus $b^n$ based on $\K^n$, so that solving the optimistic MDP would encourage the learner to explore state-actions outside $\K^n$ as well as to refine its estimates inside $\K^n$.

The policy optimization routine (in line \ref{algline:call_solver} of \cref{alg:driver}) takes these objects and returns an optimistic policy $\pi^n$. This policy $\pi^n$ updates the policy cover to $\pi_{cov}^{n+1}$, which will define the next optimistic MDP when a sufficient covariance change is made again. Over the course of learning, the optimistic MDPs gradually converge to the original MDP $\M$.

\begin{table*}[!t]
\begin{tabular}{cc}
\begin{minipage}{0.49\textwidth}
\begin{algorithm}[H]
\caption{\textsc{Solver}$(\pi_{cov},b,\K)$}
\label{alg:solver}
\begin{algorithmic}[1]

\STATE \textbf{Parameters}: $K,\eta,\kappa$
\STATE $\pi_0(\cdot \mid s ) = \text{Unif}(\mathcal A)$ if $s \in \K$ and\\  $ \pi_0(\cdot \mid s ) = \text{Unif}(\{ a \mid (s,a) \not \in \K\})$ if $s \not \in \K$ \label[line]{line:policydef}
\FOR{$k=0,1,\dots,K-1$}
\IF{$k-\underline k > \kappa$ or $k=0$}
\STATE $\underline k  \gets k$
\STATE $\mathcal D \leftarrow $ \textsc{MonteCarlo}($\pi_{cov},\pi_{k},b)$ \label[line]{line:sampler}
\ENDIF
\STATE $\Qhat_k  \gets \textsc{Critic}(\mathcal D,\pi_{\underline k},\pi_k,b)$ \label[line]{line:regression}
\STATE Update policy: $\forall s \in \K$, \\$\pi_{k+1}(\cdot \mid s) \propto \pi_{k}( \cdot \mid s)e^{\eta \Qhat_{k}( \cdot \mid s)}$ \label[line]{line:update}
\ENDFOR
\STATE \textbf{Return:} $\pi_{0:K-1} = \{\pi_0,\dots,\pi_{K-1} \}$
\normalsize
\end{algorithmic}
\end{algorithm}
\end{minipage}
&
\begin{minipage}{0.48\textwidth}
\begin{algorithm}[H]
\begin{algorithmic}[1]
\caption{\textsc{Critic}$(\mathcal D, \underline \pi,\pi,b)$}
\label{alg:ise}
\STATE \textbf{Parameters}: $W$
\FOR{$i = 1,\dots,|\mathcal D|$}
\STATE $(x_i,\mathcal P_i,G_i, b_i) \leftarrow \mathcal D[i]$ 
\STATE $\rho_i \leftarrow \Pi_{\tau=2}^{|\mathcal P_i|} \frac{\pi(a_\tau \mid s_{\tau})}{\underline \pi(a_\tau \mid s_{\tau})}$
\label{line:rho} 
\ENDFOR
\STATE \mbox{$ \widehat w = \displaystyle{\min_{\| w \|_2 \leq W}} \sum_{i=1}^{|\mathcal D|} \( x_i^\top w - \rho_iG_i - b_i  \)^2$}\label{line:fitting}
\STATE \textbf{Return:} $\Qhat(s,a) =\phi(s,a)^\top\widehat w + \frac{1}{2}b(s,a)$, $\forall s \in \K^n$ and $\Qhat(s,a) = b(s,a)$ otherwise \label{line:critic}
\end{algorithmic}
\end{algorithm}
\end{minipage}
\end{tabular}
\end{table*}

\subsection{\Alg{} : Outer loop}
\label{sec:main:alg}
Here we describe the details of three major components used in the outer loop of \Alg{} (the policy cover, the known state-actions, and the reward bonus) and our adaptive rule for updating optimistic MDPs in line~\ref{line:lazyupdate} of \cref{alg:driver}.

\paragraph{Policy Cover} At iteration $n$, we define the policy cover as $\pi_{cov}^n = \pi^{0:n-1}$, which is the uniform mixture of prior policies. 
When sampling from $\pi_{cov}^n$, we first sample $j$ uniformly from $\{0,\dots,n-1\}$ and then run $\pi^j$ to generate a trajectory. Note that, in the policy cover, the policies $\pi^j$ and $\pi^{j+1}$ differ only if invoke \textsc{Solver} in line~\ref{algline:call_solver} is invoked at the $(j+1)$th outer iteration, so the cover contains many copies of each policy. As we will discuss at the end of this section, \emph{there are only $O(d \log n)$ unique policies in the policy cover $\pi_{cov}^n$}.

\paragraph{Known state-actions}
The state-action space $\mathcal{S}\times\mathcal{A}$ is partitioned into two sets, namely the set $\K^n$ described in \eqref{eqn:known-set} of known state-actions and its complement.
When the empirical cumulative covariance matrix $\widehat{\Sigma}^n$ is significantly different from the old one (line~\ref{line:lazyupdate}), we update the known state-action set,
\begin{equation}
	\K^n = \{(s,a)  \mid \sqrt{\beta}\|\phi(s,a)\|_{(\widehat \Sigma^n)^{-1}} < 1 \}.
	\label{eqn:known-set}
\end{equation}
Intuitively, the set $\K^n$ represents the state-action pairs easily reached under $\pi_{cov}^n$, because state-action pairs with a small quadratic form lie in a direction that has a reasonable visitation under the policy cover $\pi_{cov}^n$, as noted in many prior works in linear bandits~\citep{Dani2008Stochastic,Abbasi11} and RL~\citep{jin2020provably,agarwal2020pc}. If the features for all actions at a state lie in the $\K^n$, we say the state is known; without possibility of confusion, we denote with $\K^n =\{s \mid \sqrt{\beta}\|\phi(s,a)\|_{(\widehat \Sigma^n)^{-1}} < 1, \forall a \}$ the set of known states. Unlike \pcpg, our algorithmic choices allow using a much smaller threshold $\beta$ to define a substantially larger known set, as we will see in the next section.

\paragraph{Reward bonus}
At a high level, \Alg{}  performs exploration \emph{both} in the known and unknown regions. On unknown states $\mathcal{S}\setminus \K^n$ the algorithm roughly tries to emulate \textsc{R-max} \citep{brafman2002r}, which is reasonable when the uncertainty is very high; within the known space $\K^n$, the algorithm has sufficient information to explore in a much more sophisticated and efficient way, which is enabled by the bonus described below:
\begin{align} \label{eqn:bonus}
	b^n(s,a) &= 2b^n_\phi(s,a) + b^n_\1(s,a), \quad \text{where} \\
\nonumber b_{\phi}^n(s,a) &= \sqrt{\beta}\|\phi(s,a)\|_{(\widehat \Sigma^n)^{-1}} \kn,\quad  \text{and} \quad b^n_{\1}(s,a) = \frac{3}{1-\gamma}\1\{(s,a) \not \in \K^n\}
\end{align}
In other words, the bonus is assigned differently on the known and unknown spaces.  On unknown state-actions, the assigned bonus equals $b^n_{\1}(s,a) = \tfrac{3}{1-\gamma}$, which is the largest value of the original reward over a trajectory. Consequently, visiting any such state-action pair is strictly preferable to staying within the known subset of the MDP and the known set $\K^n$ is expanded.
In the known region, the uncertainty is quantified by the bonus $b^n_\phi(s,a) = \sqrt{\beta} \|\phi(s,a)\|_{(\widehat\Sigma^n)^{-1}}$ which is the only one active inside the known space.
This form of the bonus $b^n_\phi$ is standard from the linear bandit literature (e.g., \citep{Dani2008Stochastic,Abbasi11} and linear MDPs \citep{jin2020provably}.

Our definition of bonus differs from that in the related \pcpg{} algorithm.
Unlike \Alg, \pcpg{} only explores using the bonus $b_\1$, ignoring the amount of information (or uncertainty) encoded in the quadratic form $\| \phi(s,a)\|_{(\widehat\Sigma^{n})^{-1}}$.
As a result, \pcpg{} stops exploring a state-action immediately after it becomes known (i.e. in $\K^n$).
We found that such a behavior is undesirable, because doing so would couple the threshold used in defining the known set $\K^n$ with the policy performance suboptimality $\epsilon$, ultimately resulting in a more sample inefficient exploration.

\paragraph{Adaptive updates of optimistic MDPs}
It remains to explain why infrequent or \emph{lazy updates} of the optimistic MDP (line~\ref{line:lazyupdate} of \cref{alg:driver}) are beneficial. Recall that in iteration $n$ we seek to find
 \begin{equation}
    \pi^n \approx \max_\pi \E_{s\sim \pi_{cov}^n} V^\pi(s; r+b^n),
     \label{eqn:outer_objective}
 \end{equation}
and add it to the policy cover. However, finding this policy entails a significant sample complexity because the \textsc{Solver} (\cref{alg:solver}) --- which relies on Monte Carlo estimations to evaluate its policies --- must be invoked.

This suggests to call the \textsc{Solver} only when the returned policy $\pi^n$ is expected to be significantly better than the prior one $\pi^{\underline n}$ or to make a significant contribution to the policy cover.
Because the optimistic MDP is defined by the bonus $b^n$, which is a function of $\widehat{\Sigma}^n$ (see \eqref{eqn:known-set} and~\eqref{eqn:bonus}), the optimistic MDP only changes significantly when the updated $\widehat{\Sigma}^n$ is very different as measured by its determinant. Therefore, each time the determinant doubles (line~\ref{line:lazyupdate}), we update the known set and the bonus according to \eqref{eqn:known-set} and \eqref{eqn:bonus}, respectively, based on the latest $\widehat{\Sigma}^n$. Then we invoke the \textsc{Solver} to find a new policy to update the policy cover. As a result, the number of solver invocations is reduced from $O(N)$ to $O(d\log N)$, providing substantial sample complexity gains.

\subsection{\Alg: Inner loop}
\label{sec:main:solver}
We now turn our attention to the \textsc{Solver}, \cref{alg:solver}. At a high-level, we initialize the policy to be a uniform distribution that prefers at a state $s$ to take an unknown action $a$ such that $(s,a) \notin \K^n$, and employ the online learning algorithm (the exponentiated weight update~\citep{freund1997decision}) on the \emph{known} states to update the policy.
This update rule is equivalent to the Natural Policy Gradient (NPG) algorithm for log-linear policies~\citep{kakade2001natural,agarwal2020optimality}.

The update rule is an actor-critic scheme, where we fit the critic by regressing on the observed Monte Carlo returns and update the actor using exponentiated weights. As argued in \citep{agarwal2020optimality}, using Monte Carlo as critic is an essential technique to provide better robustness to model misspecification compared with a least squares policy evaluation (LSPE)~\citep{bertsekas1996temporal} method, but is also a significant source of  sample complexity.

\paragraph{Fitting the critic with nearly on-policy data}
To improve the sample complexity of \cref{alg:solver}, we devise a way to \emph{reuse past data while keeping the robustness property of Monte Carlo}.

Our estimator reuses data by applying trajectory-level importance sampling on past Monte Carlo return estimates~\citep{precup2000eligibility}. While trajectory-level importance sampling has been typically associated with exponentially high variance, we found that its variance is constant when we properly control how much into the past the data are reused, because the policies produced by the online learning here do not change significantly between successive updates but induce similar trajectories.

At iteration $k$ in \cref{alg:solver}, we have access to a dataset of trajectories previously drawn in \cref{alg:Sampler} by first sampling $s,a\sim \pi_{cov}$, and then following the policy $\pi_{\underline k}$ for some prior iteration $\underline k \geq k - \kappa$ (see \cref{alg:Sampler} in the appendix for details). We use this dataset to obtain a Monte Carlo return estimate for the current policy $\pi_k$ by reweighting the samples with importance sampling (see \cref{alg:ise} for details). Subsequently, we learn a critic by training a linear function to map $\phi(s,a)$ --- the feature vectors for the initial state and action sampled from $\pi_{cov}$ --- to the reweighted random return 
via least squares linear regression (line~\ref{line:regression} in \cref{alg:solver} and \cref{alg:ise}).

Following prior works~\citep{jin2020provably,agarwal2020pc}, we offset the sampled return by the bonus value at the initial state-action in the trajectory in line~\ref{line:fitting} of \cref{alg:ise}. This offset ensures that the regression target is perfectly realizable using a linear function in $\phi(s,a)$ when the MDP is exactly linear, despite the non-linear bonus function.

\paragraph{Cautious optimism and one-sided errors} Since the critic fitting in line~\ref{line:fitting} of \cref{alg:ise} is offset by the initial bonus to preserve linearly of the representation, it would be natural to define the critic estimates as \mbox{$\what{Q}(s,a) = \what{w}^\top\phi(s,a) + b^n(s,a)$}, which would exactly correct for the offset (this is the approach taken in \pcpg).
However, in line~\ref{line:critic} of \cref{alg:ise}, we only partially correct for the offset and instead define the critic estimate as \mbox{$\what{Q}(s,a) = \what{w}^\top\phi(s,a) + \textcolor{red}{\frac{1}{2}}b^n(s,a)$}. This introduces a \emph{negative bias} in the estimate. However, since our critic is being fit to the bonus augmented returns, we are able to show in our analysis that $\what{Q}_k(s,a)$ (in line \ref{line:regression} of \cref{alg:solver}) is still optimistic relative to $Q^{\pi_k}(s,a;r)$, while being an underestimate of $Q^{\pi_k}(s,a;r+b^n)$. This one-sided error property plays a crucial role of improving a factor of $O(\frac{1}{\epsilon})$ in sample complexity.

\paragraph{Actor updates} With the critic computed above, line~\ref{line:update} in \cref{alg:solver} updates the policy on the known states using the exponentiated weight updates, with the critic function as the negative loss. We change the data collection policy every $\kappa$ iterations to collect a fresh dataset for critic fitting.

\section{Main Result}
In this section we provide the main guarantees for \Alg{}. We make the following \emph{transfer error} assumption, originally introduced in \citep{agarwal2020optimality} for policy gradient algorithms.

\begin{definition}[Transfer Error]
\label[definition]{def:TransferError}
Define the loss functional
\begin{align}
\label{eqn:L}
\mathcal L(w,d,f) \defeq \frac{1}{2} \E_{(s,a) \sim d}\Big[\phi(s,a)^\top w - f \Big]^2.
\end{align}
For a given outer iteration $n$ (in \cref{alg:driver}) and an inner iteration $k$ (in \cref{alg:solver}) let
\begin{align}
	Q^n_k(s,a) = Q^{\pi_k}(s,a;r+b^n), \quad Q^n(s,a) = Q^{\pi^n}(s,a;r+b^n)
\end{align}
be the optimistic action-value functions. Define the `best' regression parameters
\begin{align}
w^{n,\star}_k \in \argmin_{\|w\|_2\leq W} \mathcal L(w,\rho^n,Q^n_k - b^n), \qquad w^{n,\star} \in \argmin_{\|w\|_2\leq W} \mathcal L(w,\rho^n,Q^n - b^n).
\end{align}
Then the transfer error with respect to a fixed comparator $\pitilde$ is defined as\footnote{Shifting the $Q$ values below by the bonus $b^n$ in regression and adding the bonus afterwards is a standard practice in exploration methods (see, e.g., \citep{jin2020provably}).}
\begin{align}
\mathcal E_k^n \defeq \mathcal L(w^{n,\star}_k,d^{\pitilde} \circ \text{Unif}(|\ActionSpace|) ,Q^n_k - b^n).
\end{align}
For compactness we denote the average approximation error across $N$ and $K$ (the inner and outer iterations of the algorithm) as $\sqrt{\mathcal E} \defeq \frac{1}{NK}\sum_{n=1}^N \sum_{k=0}^{K-1}\sqrt{\mathcal E^n_k}$.
\end{definition}

The transfer error measures the average prediction error of the agent's estimator in the limit of infinite data on unseen samples. For the transfer error to be small, the estimator does not need to be pointwise accurate but only accurate in \emph{expectation} along a \emph{fixed distribution}, namely the state-action distribution induced by the comparator $\pitilde$ (typically the optimal policy $\pistar$). These are substantially weaker requirements than the typical $\ell_\infty$ error assumption arising from the use of temporal difference methods. In particular, on the low-rank or linear MDP model \citep{yang2020reinforcement,jin2020provably,zanette2020frequentist} the transfer error is zero. In this case, we say that the linear model is not misspecified; for more details please see \cref{sec:Notation}.

\begin{theorem}[Sample Complexity Analysis of \Alg{} ]
\label{thm:main:MainResult}
Fix a failure probability $\delta$; for appropriate input parameters, $$\(N,K,\eta,\lambda,\kappa,W\) = \widetilde O\( \frac{d^2}{(1-\gamma)^8\epsilon^2},\frac{\ln|\mathcal A|W^2}{(1-\gamma)^2\epsilon^2},\frac{\sqrt{\ln |\ActionSpace|}}{\sqrt{K}W},d,\frac{1-\gamma}{\eta W},\frac{1}{(1-\gamma)^2}\)	,
$$ \Alg{}  returns with probability at least $1-\delta$ a policy $\pi^{\Alg{} }$ such that
 \begin{align*}
 	\( \Vstar - V^{\pi^{\Alg{} }} \)(s_0) \leq \epsilon + \frac{2\sqrt{2|\ActionSpace|\mathcal E}}{1-\gamma}, \quad \quad \quad
 \end{align*}
 using at most $\widetilde {O}(\frac{d^3}{(1-\gamma)^{13}\epsilon^3})$ samples.
 \end{theorem}

We now discuss some aspects of our result and compare it to the most relevant prior works.
\paragraph{Better robustness compared to LSVI-UCB} Compared to \citep{jin2020provably} on well-specified linear MDPs, \Alg{}  provides PAC bounds to find an $\epsilon$-optimal policy and inherents the same $O(d^3)$  dependence on the feature dimension as \citep{jin2020provably}, while being $1/\epsilon$ worse in sample complexity and in horizon dependence; a $\log|\ActionSpace|$ factor is also implicitly hidden in our $\widetilde O$ notation. However, the transfer error of \Alg{} in \cref{def:TransferError} can be a significantly weaker assumption, as discussed in~\citep{agarwal2020pc}.

\paragraph{Sample complexity improvement relative to \pcpg} \Alg{}  operates under an essentially identical notion of transfer error as \pcpg{} in \citep{agarwal2020pc} and shares several \pcpg's algorithmic principles (e.g. the exponentiated weights rule for policy update, Monte Carlo for policy evaluation, and the concept of policy cover).
But importantly, because \Alg{} uses a better bonus structure, adaptive bonus updates, and performs importance sampling to reuse Monte Carlo data, \Alg{}  is able to lower the sample complexity from the slow $ \widetilde{O}(\frac{1}{\epsilon^{11}})$ rate of \pcpg{} to the faster $\widetilde{O}(\frac{1}{\epsilon^3})$ rate.
Note that unlike \pcpg, we do not extend our analysis to the infinite dimensional setting, though we expect it to be possible using the covering arguments from \citet{yang2020bridging}.

\paragraph{Better sample complexity in the optimization setting} Finally, \Alg{} 's analysis is based on the natural policy gradient algorithm \citep{kakade2001natural}, which has recently been analyzed in \citep{agarwal2020optimality} when a good sampling distribution is already given  (for example, through a generative model). For solving the policy optimization subproblem, \Alg{}  improves the $\widetilde{O}(\frac{1}{\epsilon^4})$ rate obtained in \citep{agarwal2020optimality} to $\widetilde{O}(\frac{1}{\epsilon^3})$ by the data reuse scheme described in \cref{sec:main:solver}.

\section{Technical Analysis}
In this section we briefly sketch the analysis of \Alg{}  and prove \cref{thm:main:MainResult}. We start by giving a regret decomposition analysis of the policies computed by \Alg{}  in \cref{sec:main:RegretDecomoposition}. This result will be used as the foundation of the proof of the main result in \cref{sec:main:SampleComplexityAnalysis}

\paragraph{Notation}
We introduce a few more notations to simplify the presentation. The outer policy $\pi^n$ is a uniform mixture of the policies $\pi^n_0,\dots,\pi_{K-1}^n$ returned by the \textsc{Solver} (see \cref{alg:solver}) in outer iteration $n$. For the policy $\pi^n$ in the outer iteration $n$ in \cref{alg:driver}, we denote with $Q^n(s,a) = Q^{\pi^n}(s,a;r+b^n)$ the state-action value function, with $V^n(s) = V^{\pi^n}(s;r+b^n)$ the state value function, and with  $A^n(s,a) = Q^n(s,a)-V^n(s)$ the advantage function on the optimistic MDP.
Similarly, for the linear approximation (given by the Monte Carlo regression in line \ref{line:regression} of \cref{alg:solver}), we write  $\Qhat^n(s,a) =\frac{1}{K}\sum_{k=0}^{K-1}\Qhat^{\pi^n_k}(s,a)$,  $\widehat{V}^n(s)=\frac{1}{K}\sum_{k=0}^{K-1}\Qhat^{\pi^n_k}(s,\pi^n_k)$, and $ \Ahat^n(s,a) = \Qhat^n(s,a)-\Vhat^n(s)$.
Using the best regressed parameter $w^{n,\star}$ in \cref{def:TransferError}, we also define the best predictor $Q^{n,\star}(s,a) = \phi(s,a)^\top w^{n,\star}+b^n(s,a)$ and its advantage function $A^{n,\star}(s,a) = Q^{n,\star}(s,a)-V^{n,\star}(s)$. In absence of misspecification, we note that $Q^{n,\star} = Q^n$.

\subsection{Regret Decomposition}
\label{sec:main:RegretDecomoposition}
Fix an outer iteration index $n$.
We start the analysis by giving the following performance lemma, which
is obtained by combining the performance difference lemma~\citep{kakade2002approximately} with several properties of our algorithm.
\begin{lemma}[Performance Analysis; \eqref{eqn:RefInMain} in appendix]
\label[lemma]{prop:main:MainAnalysisIntermediate}
With high probability, \Alg{}  ensures
\begin{align*}
	(1-\gamma)(\Vstar - V^{\pi^n})(s_0) & \leq  \underbrace{\sup_{s\in\K^n}\Ahat^{n}(s,\pistar)}_{\substack{\text{Solver error}}}
	+ \E_{(s,a) \sim \pistar}\underbrace{\Big| A^{n}(s,a)-A^{n,\star}(s,a)\Big|\kn}_{\substack{\text{Approximation error on states in $\K^n$}}} \\
	&+ \E_{(s,a) \sim \pistar} \underbrace{\Big(Q^{n,\star}(s,a) - \Qhat^{n}(s,a)\Big)\kn}_{\substack{\text{Statistical error along $\pistar$ on states in $\K^n$}}}   \\
	 & - \underbrace{\E_{(s,a) \sim \pistar} 2b_{\phi}^n(s,a)\kn}_{\substack{\text{Bonus along $\pistar$ on states in $\K^n$}}} + \underbrace{\E_{(s,a) \sim \pi^n} b^n(s,a)}_{\substack{\text{Bonus  along $\pi^n$ on the full space }}}.
	\numberthis{\label{eqn:main:MainAnalysisIntermediate}}
\end{align*}
\end{lemma}
We discuss each of these terms in detail below.

\subsubsection{Solver error}
The first term in \eqref{eqn:main:MainAnalysisIntermediate} measures how well the policy $\pi^n$ performs in terms of our empirical advantage function on known states; generating such a policy is done using the regret guarantee of our online learning rule in~\cref{alg:solver}.
We have the following lemma (see also \citep{agarwal2020optimality,agarwal2020pc}).

\begin{lemma}[Online regret of softmax; \cref{lem:NPG} in appendix]
\label[proposition]{prop:main:npg}
Using an appropriate learning rate $\eta$,  \cref{alg:solver} identifies a mixture policy $\pi^n$ that satisfies
$
	\sup_{s\in \K^n} \sup_{a\in\mathcal{A}} \Ahat^n(s,a) =  \widetilde O\(\frac{1}{(1-\gamma)^2}\sqrt{\frac{1}{K}}\).
$
\end{lemma}
Thus the solver error in \eqref{eqn:main:MainAnalysisIntermediate} can be reduced arbitrarily, although the number of iterations $K$ directly affects the sample complexity; see \cref{sec:main:SampleComplexityAnalysis}.

\subsubsection{Approximation error}
The second term in \eqref{eqn:main:MainAnalysisIntermediate} is an approximation error in advantages under $\pi^\star$ and is non-zero only when the linear MDP assumption is not exactly satisfied. The performance bound of \cref{prop:main:MainAnalysisIntermediate} highlights that the \emph{approximation error} is measured  \textit{1)} in expectation and \textit{2)} along the distribution induced by $\pistar$. For brevity, we neglect the approximation error in this proof sketch; we note that this quantity can be controlled in the general version of the result using the transfer error condition (\cref{def:TransferError}).

\subsubsection{Statistical error}
The third term in \eqref{eqn:main:MainAnalysisIntermediate} is perhaps the most surprising: it reasons about the statistical error in our critic fitting on the known states, \emph{but only under states and actions chosen according to $\pi^\star$}.
In other words, the agent's estimator does not need to be correct for arbitrary distributions; otherwise, an $\ell_\infty$ guarantee over the known-set is needed (as needed by~\citet{agarwal2020pc}).
Such result is enabled by the following key lemma, which contributes the underestimation property of $\Qhat$ needed in the proof of \cref{prop:main:MainAnalysisIntermediate}. (Recall the regression target is subtracted with $b^n(s,a)$ but the final predictor adds back only $\frac{1}{2}b^n(s,a)$.)

 \begin{lemma}[One sided errors; \cref{lem:ValidityOfConfidenceItervals} in appendix]
 \label[lemma]{lem:main:sided}
 Let $w^{n,\star}$ be defined in \cref{def:TransferError} and let $\widehat w^n$ be the corresponding empirical minimizer.  Define the agent's predictor on $(s,a) \in \K^n$ as
 	$\Qhat^{n,\star}(s,a)  \defeq \phi(s,a)^\top \widehat w^{n} + \frac{1}{2}b^n(s,a)$.
Then with high probability, 
jointly $\forall n$ and $\forall (s,a) \in \K^n$,
	\begin{align}
	\quad 0 \leq (Q^{n,\star} - \Qhat^{n}) (s,a)  \leq b^n(s,a) = 2b_\phi(s,a).
\end{align}
\end{lemma}

\subsubsection{Bonus difference and concentration}
The final two terms in \eqref{eqn:main:MainAnalysisIntermediate} arise as we optimize policies in the optimistic MDP, but the performance difference of interest is defined for the original MDP. The negative bonus term under the comparator $\pi^\star$ helps cancel some of the statistical errors (i.e. the third term), which is crucial for the overall sample complexity results.

For the other bonus term under $\pi^n$, we can bound it using the elliptic potential lemma (e.g., \citep{Abbasi11}) and a martingale argument.
\begin{lemma}[Concentration on Bonus; \cref{lem:SumOfIndicators,lem:SumOfBonuses} in appendix]
	With high probability, it holds that
	$
	\sum_{n=1}^N \E_{(s,a) \sim \pi^n} b^n(s,a) = \widetilde O\(\frac{d\sqrt{N}}{(1-\gamma)^3}\)
	$.
\end{lemma}

\subsection{Sample Complexity Analysis (Proof of theorem \ref{thm:main:MainResult})}
\label{sec:main:SampleComplexityAnalysis}

In order to bound the sample complexity for obtaining Theorem~\ref{thm:main:MainResult}, we need to bound the following quantities:
\begin{enumerate} 
	\item the number of outer iterations $N$ to control the number of samples collected for the matrix $\widehat{\Sigma}^n$,
	\item the number of calls to \textsc{Solver} across $N$ iterations,
	\item the number of data collection rounds in \textsc{Solver} for critic fitting, and
	\item the number of samples in each dataset that \textsc{Solver} collects. We start with bounding the number of inner and outer iterations.
\end{enumerate}

\begin{lemma}[Convergence rate of \Alg{} ; \cref{prop:MainAnalysis} in appendix]
\label[lemma]{prop:main:MainAnalysis}
	With high probability \Alg{} , computes policies $\pi_1,\dots,\pi_N$ such that
	\begin{align*}
	\frac{1}{N}\sum_{n=1}^N\( \Vstar - V^{\pi^n} \)(s_0) \leq \underbrace{ \widetilde O\( \frac{1}{(1-\gamma)^3\sqrt{K}} \)}_{\text{Solver error}} + \text{Approx. error} + \underbrace{\widetilde O\(\frac{d}{(1-\gamma)^4\sqrt{N}} \)}_{\text{Average statistical uncertainty}}.
	\end{align*}
	where \emph{Approx. error} denotes the second term in \cref{prop:main:MainAnalysisIntermediate}.
\end{lemma}
Using the above proposition we can give a proof of \cref{thm:main:MainResult}.
\begin{proof}(of \cref{thm:main:MainResult})
To ensure the average suboptimality gap is below $\epsilon$ we need to ensure:
\begin{align}
\label{eqn:main:KN}
	\frac{1}{N}\sum_{n=1}^N\( \Vstar - V^{\pi^n} \)(s_0) \leq \epsilon \qquad \longrightarrow \qquad K \approx \frac{1}{(1-\gamma)^6\epsilon^2} , \quad N \approx \frac{d^2}{(1-\gamma)^8\epsilon^2}.
\end{align}
Next we bound the number of calls to \textsc{Solver}; this is controlled by the lazy update (line \ref{line:lazyupdate} in \cref{alg:driver}) and the bonus structure.
\begin{lemma}[Number of solver calls; \cref{lem:Switches}]
\Alg{} invokes \textsc{Solver} at most $O(d\log N)$ times.
\end{lemma}
Every time it is invoked, \textsc{Solver} runs for $K$ iterations in \eqref{eqn:main:KN} and at every iteration it needs to receive an evaluation on the performance of the current policy ($\Qhat_k$ estimator from the critic, \cref{alg:ise}). Using importance sampling we can avoid collecting fresh Monte Carlo data for every policy. We control the number of data collection rounds based on importance sampling variance.
\begin{lemma}[Stability of the Importance Sampling Estimator; \cref{lem:bias-variance} and \eqref{eqn:kappa} in appendix]
The importance sampling ratio used in \cref{alg:ise} is bounded by a constant with high probability:
\begin{align*}
\text{If} \quad k - \underline k =  \widetilde O\(\sqrt{K}(1-\gamma)\), 
\quad  \text{then}  \quad \Pi_{\tau=2}^t \frac{\pi_k(s_\tau,a_\tau)}{\pi_{\underline k}(s_\tau,a_\tau)} \leq 2, \qquad \forall \{s_1,a_1,\dots,s_t,a_t \}.
\end{align*}
\end{lemma}
In other words, after fresh Monte Carlo trajectories are collected, the importance sampling estimator can be used to make stable predictions of the value roughly for the future $\sqrt{K}(1-\gamma)$ policies.
This implies that we need to collect fresh data at most once every $K / ((1-\gamma)\sqrt{K})= \widetilde O\( \frac{\sqrt{K}}{1-\gamma}\)$ iterations.

It remains to specify the number of samples we collect in each round of data collection. Note that in our statistical analysis, we want the critic error to be bounded by $b^n(s,a)$, which roughly goes down as $O(1/\sqrt{n})$, as the matrix $\widehat{\Sigma}^n$ that defines the bonus grows linearly in $n$. This vague intuition can be formalized by appealing to standard linear regression analysis to show that we need to collect $O(n)$ Monte Carlo returns to fit the critic in outer iteration $n$.
 \begin{lemma}[Number of Monte Carlo Trajectories]
	When the Monte Carlo procedure is invoked at the outer iteration $\underline n$, at most $\underline n \leq N$ trajectories are collected.
 \end{lemma}

Finally, $\widetilde O(\frac{\log(1/\delta)}{1-\gamma})$ is a uniform high probability bound on the length of each Monte Carlo trajectory, which implies the total sample complexity of \Alg{}  is
 \begin{align*}
\underbrace{\vphantom{\widetilde O\(\frac{\sqrt{K}}{(1-\gamma)}\)}\widetilde O(d) \vphantom{\widetilde O\(\frac{1}{\epsilon^2}\)}}_{\text{\# calls to \cref{alg:solver}}}
	\times 	\underbrace{\widetilde O\(\frac{\sqrt{K}}{1-\gamma}\)}_{\text{\# calls to \cref{alg:Sampler}}}
	\times \underbrace{\widetilde O\vphantom{\widetilde O\(\frac{\sqrt{K}}{1-\gamma}\)} \(N\)}_{\text{\# Monte Carlo trajectories}}
	\times \underbrace{\widetilde O\vphantom{\widetilde O\(\frac{\sqrt{K}}{(1-\gamma)}\)} \(\frac{1}{1-\gamma}\)}_{\text{\# samples per trajectory}}
	= \quad \widetilde O\( \frac{d^3}{(1-\gamma)^{13}\epsilon^3} \).
\end{align*}
\end{proof}

\section{Discussion}

In this paper, we advance the theoretical understanding of sample-efficient policy optimization methods with strategic exploration and robustness to model misspecification. While we carry out our analysis for a specific algorithm, we expect the insights developed here for sample complexity improvements to be more broadly applicable. 
For instance, the exponentiated weight updates in our policy optimization subroutine can generally be substituted with other no-regret algorithms from the Follow The Regularized Leader family. As usual, we expect different choices to offer varying trade-offs in their dependence on problem parameters; with reasonable choices, they are still amenable to the importance sampling based data reuse.
Similarly, the lazy updates for the bonus are generically applicable. Note that our algorithmic choices strike a particular balance of a very infrequent bonus update and a fairly accurate optimization. Prior works in different, but related problems~\citep{agarwal2014taming} have shown that often there is flexibility in these choices, such as more regular updates followed by coarser optimization, which might be empirically preferable.

Perhaps the most important outstanding question not addressed here is how to close the gap between the $\widetilde{O}(1/\epsilon^2)$ sample complexity that is known to be achievable in linear MDPs (see e.g.~\citet{jin2020provably}) and our worse dependence of $\widetilde{O}(1/\epsilon^3)$.
There appears to be a trade-off in terms of the allowable assumption on model misspecification, and approaches based on Least Square Policy Evaluation for data reuse (such as LSVI-UCB) fail to work under our transfer error assumption and the special cases in~\citet{agarwal2020pc}. Whether this trade-off is fundamental, or if a single method can be developed to be robust to transfer error, while enjoying an optimal sample complexity guarantee in the absence of misspecification is an interesting direction for future work.

\section*{Acknowledgment}
Most of the work was completed while Andrea Zanette was interning at Microsoft Research and the remaining part of the work was done while Andrea Zanette was visiting the Simons Institute for the Theory of Computing.

\acks{The authors are grateful to the reviewers for their helpful comments.}

\bibliography{rl}

\appendix

\onecolumn
\newpage
\section{Remaining Algorithm Components}
\begin{algorithm}[H]
	\caption{\textsc{MonteCarlo}($\pi_{1:q},\pi,b$)}
	\label{alg:Sampler}
	\begin{algorithmic}[1]
	\STATE \textbf{Inputs}: Policy cover $\pi^{1:q}$, evaluation policy $\pi$, additional reward $b$
	\STATE $\mathcal D = \emptyset$
	\FOR{iteration $i = 1,\dots,q$}
	\STATE Sample $j$ uniformly at random in $[q]$
	\STATE Sample $\tau\geq 1$ with probability $\gamma^{\tau-1}(1-\gamma)$
	\STATE Execute $\pi_j$ for $\tau-1$ steps from a sampled initial state, giving state $s$
	\STATE Sample action $a \sim \pi_j(\cdot \mid s)$
	\STATE Sample $h \geq 1$ with probability $\gamma^{h-1}(1-\gamma)$
	\STATE Continue the rollout from $(s,a)$ by executing $\pi$ for $h-1$ steps, \\giving the rollout $\mathcal P = \{(s_1,a_1,\dots,s_h,a_h)\}$ where $(s_1,a_1) = (s,a)$
	\STATE $G = \frac{1}{1-\gamma}[r(s_h,a_h) + b(s_h,a_h)]$
	\STATE $\mathcal D \leftarrow \mathcal D \cup \{(\phi(s,a),\mathcal P,G, b(s,a)) \} $
	\ENDFOR
	\STATE \textbf{return} $\mathcal D$
	\normalsize
	\end{algorithmic}
\end{algorithm}

\begin{algorithm}[H]
\caption{\textsc{FeatureSampler}($\pi$)}
	\label{alg:FeatureSampler}
	\begin{algorithmic}[1]
	\STATE Sample $\tau \geq 1$ with probability $\gamma^{\tau-1}(1-\gamma)$
	\STATE Execute $\pi$ for $\tau-1$ steps from a sampled initial state, giving state $s$
	\STATE Sample action $a \sim \pi(\cdot \mid s)$
	\STATE \textbf{return} $\phi(s,a)$
	\normalsize
	\end{algorithmic}
\end{algorithm}

\newpage
\section{Additional Related Literature}
\label{sec:Literature}
Exploration has been widely studied in the tabular setting \citep{Azar17,zanette2019tighter,efroni2019tight,jin2018q,dann2019policy,zhang2020reinforcement,russo2019worst}, but obtaining formal guarantees for exploration with function approximation is a challenge even in the linear case due to recent lower bounds \citep{du2019good,weisz2020exponential,zanette2020exponential,wang2020statistical}. When the action-value function is only approximately linear, several ideas from tabular exploration and linear bandits \citep{lattimore2020bandit} have been combined to obtain provably efficient algorithms in low-rank MDPs  \citep{yang2020reinforcement,zanette2020frequentist,jin2020provably} and their extensions \citep{wang2019optimism,wang2020provably}. Minimax regret bounds for under little or zero inherent Bellman error (a superset of low-rank MDPs) is given in \citep{zanette2020learning} and a computationally tractable algorithm for that setting has recently been proposed \citep{zanette2020provably}. No inherent Bellman error is a subset of a more general framework of MDPs with low Bellman rank \citep{jiang17contextual} where the inherent Bellman error is allowed to have a low rank structure but no computationally tractable algorithm are known for such general setting \citep{dann2018oracle}.

Extensions of the linear or low-rank MDP models to kernel and neural function approximation have recently been presented in \citet{yang2020function}. Other linear transition models recently considered include those presented by \citep{ayoub2020model,zhou2020provably}; for the latter, a minimax algorithm has recently been proposed \citep{zhou2020nearly}.

If linearity holds only for the optimal action-value function  and one is only interested in identifying an optimal policy (as opposed to a near optimal one), then \citep{du2020agnostic} provide an algorithm for such setting, although a sample complexity proportional to the inverse gap (which can be exponentially small) must be suffered. Deterministic systems with linear value functions are also learnable  in finite horizon by just assuming realizability \citep{WR13}.

Finally there is a rich literature on the convergence properties of policy gradient methods \citep{kakade2002approximately,azar2011dynamic,scherrer2014local,neu2017unified,even2009online,geist2019theory,liu2019neural,abbasi2019politex,bhandari2019global,fazel2018global,agarwal2020optimality} although these do not address the exploration setting. Notable exceptions include:  \citep{shani2020optimistic} on tabular domains and \citep{cai2020provably} on a linear MDP model different than the one we consider here and the aforementioned work of~\citet{agarwal2020pc}.

\newpage
\section{Additional Notation and MDP Construction}
\label{sec:Notation}
In table \cref{tab:MainNotation} we define some frequently used symbols that will be used in the following analyses.
\renewcommand{\arraystretch}{1.5}
\begin{longtable}{l l p{9cm}}
\caption{Symbols}\\
\hline
$B$ & $ \defeq $ & $\frac{3}{1-\gamma}$ \\
$G_{max}$ & $ \defeq $ & $\frac{2+B}{(1-\gamma)}$ \\
$W$ & $ \defeq $ & $2 G_{max}$  \\
$\kappa$ & $\defeq$ & see \cref{eqn:kappa} \\
$\lambda$ & $\defeq$ &  $\lambda_{min}$,  see \cref{eqn:lambdamin} \\
$\beta$ & $\defeq$ &  see \cref{eqn:beta} \\
$\mathcal B$ & $\defeq$ & $\{ v \in \R^d \mid \| v \|_2 \leq 1 \}$ \\
$t_{max}$ & $\defeq$ & $ \frac{\ln(16N^2K/\delta)}{1-\gamma}$ (maximum  high probability trajectory length \fullref{lem:TrajectoryBoundness})
\label{tab:MainNotation}
\end{longtable}

We denote with $n$ the \emph{outer} iterations (see \cref{alg:driver}) and with $k$ the \emph{inner} iterations (see \cref{alg:solver}). We use the outer iteration index $n$ as superscript and the inner iteration index $k$ as subscript to indicate that a certain quantity that is computed in the outer iteration $n$ and the inner iteration $k$, respectively.

\paragraph{Transfer error on linear MDPs}
On linear MDPs, the transfer error in \cref{def:TransferError} is exactly zero, i.e., $\mathcal E = 0$. This follows by combining Claim D.1 with Lemma D.1 in \citep{agarwal2020pc}.

\paragraph{Average policy and cover}
In the analysis we use the concept of average policy or policy mixture.
\begin{definition}[Average Policy]
\label[definition]{def:PolicyMixture}
	Given policies $\pi^0,\dots,\pi^{n-1}$ let the average policy  $\pi^{0:n-1}$ be defined as follows: sample $i \in \{0,1,\dots,n-1 \}$ with uniform probability and the follow $\pi^i$ for the episode.
\end{definition}

Let $d^{\pi}$ be the distribution over state-actions induced by policy $\pi$, and let $\rho^n_{cov} = \frac{1}{n} \sum_{i=0}^{n-1} d^{\pi_i}$ be that induced by $\pi^{0:n-1}$.

\paragraph{Remark on expressions containing mixture policies}
We highlight that when the mixture policy $\pi^n$ appears in an expression, for notational convenience it is intended that the whole expression is averaged. For example, when writing the expected bonus $\E_{s \sim \pi^n} b(s,\pi^n) = \frac{1}{K}\sum_{k=0}^{K-1}\E_{s \sim \pi^k_n}b(s,\pi^n_k)$. This is to be consistent with the way the mixture policies are defined (and the way the algorithm operates), where an index $j$ in $\{0,\dots,K-1\}$ is sampled uniformly at random and then policy $\pi_j$ is followed for the full episode; to be consistent, all quantities must then refer to the same policy $\pi_j$, for example $\widehat A^n(s,a) = \widehat Q^n(s,a) - \widehat Q^n(s,\pi^n) = \sum_{k=0}^{K-1}\( \widehat Q^{n}_j(s,a) - \widehat Q^{n}_j(s,\pi^n_{j})\) = \sum_{k=0}^{K-1}\( \widehat Q^{n}_j(s,a) - \widehat V^{n}_j(s)\)$.

\paragraph{Known states}
We define the set of known state-actions in a certain outer episode $n$ (this stays constant for all inner iterations $k$ of \cref{alg:solver} as $n$ is fixed) the following set (we overload the notation as there is no possibility of confusion)

\begin{align}
\K^n & \defeq \Big\{ s\in \StateSpace  \mid \forall a\in \ActionSpace,  \sqrt{\beta}\|\phi(s,a)\|_{(\widehat \Sigma^n)^{-1}} < 1 \Big\} \\
\K^n & \defeq \Big\{ (s,a)\in \StateSpace\times \ActionSpace \mid   \sqrt{\beta}\|\phi(s,a)\|_{(\widehat \Sigma^n)^{-1}} < 1 \Big\}.
\end{align}

\paragraph{Inner policies}
The inner policies $\pi_0,\pi_1,\dots$ are those computed by \cref{alg:solver} and are defined as:
\begin{align} \label{eq:policy definition}
\forall s \in \K^n: & \quad \quad \quad \pi_{k+1}(\cdot \mid s) \propto \pi_{k}( \cdot \mid s)e^{\eta \Qhat_{k}( \cdot \mid s)} \\
\forall s \notin \K^n: & \quad \quad \quad \pi_{k+1}(\cdot \mid s) =  \text{Unif}(\{ a \mid (s,a) \not \in \K^n\})
\end{align}
The initialization is
\begin{align}
\forall s \in \K^n: & \quad \quad \quad  \pi_0(\cdot \mid s ) = \text{Unif}(\mathcal A)  \\
	\forall s \notin \K^n: & \quad \quad \quad \pi_0(\cdot \mid s )  = \text{Unif}(\{ a \mid (s,a) \not \in \K^n\})
\end{align}
\paragraph{Outer policies}
The outer policies $\pi^1,\pi^2,\dots$ are those maintained by \cref{alg:driver} and they are a mixture of the inner policies computed by the \textsc{Solver}. In particular, when the \textsc{Solver} terminates it returns a mixture of policies $\pi_{0:K-1}$, and $\pi^n$ is set to be equivalent to that mixture.
\paragraph{Bonus and optimistic MDP}
Consider $\M = (\StateSpace,\ActionSpace,p,r,\gamma)$.
In iteration $n$ we construct an optimistic MDP with bonus $b^n: \StateSpace \times \ActionSpace \rightarrow \R$ defined as $\M^n = (\StateSpace,\ActionSpace \cup \{a^\dagger\},p,r + b^n,\gamma)$.
The bonus function reads as
\begin{align}
b^n_{\phi}(s,a) & \defeq \sqrt{\beta}\|\phi(s,a)\|_{(\widehat \Sigma^n)^{-1}} \kn \\
b^n_\1(s,a) & \defeq \frac{3}{1-\gamma}\1\{(s,a) \not \in \K^n \} \\
b^n(s,a) & \defeq 2b^n_\phi(s,a) + b^n_\1(s,a).
\end{align}

The bonus behaves as follows. In any state, if $s \in \K^n$ then $b^n(s,a) = 2b^n_\phi(s,a)$ and if $(s,a) \not \in \K^n$ then $b^n(s,a) = b^n_\1(s,a)$. In particular, the bonus are in the range $[0,2)$ if the state-action is known, and otherwise the bonus is deterministically set to $\frac{3}{(1-\gamma)}$.
Notice that a state-action $(s,a)$ such that $s\notin\K^n$ but $(s,a)\in\K^n$ has zero bonus (or generally we can set an arbitrarily value here); the specific bonus value at such a state-action is irrelevant as the algorithm's policy $\pi^n$ by construction (cf. \cref{eq:policy definition}) always takes an action with the indicator bonus $b^n_\1(s,a)$ if the state $s\notin \K^n$.

The optimistic MDP has an extra action $a^\dagger$ that self loops in the current state with probability $1$ with a reward $r(s,a^\dagger) =3$. The bonus function $b^n(s,a^\dagger) = 0$. (The agent is not even aware of the existence of $a^\dagger$; this extra action  $a^\dagger$ is introduced purely for analysis.)
Denote the state-action value function of a generic policy $\pi$ on $\M^n$ with $Q^{n,\pi}(s,a) = \frac{1}{1-\gamma}\E_{(s',a')\sim\pi \mid (s,a)} [r(s',a') + b^n(s',a')]$.
The state value function is denoted with $V^{n,\pi}(s) = \E_{a\sim \pi(\cdot \mid s)}Q^{n,\pi}(s,a)$.

Let $\pi^n$ be the policy identified by the agent in the outer episode $n$ (the policy returned by \cref{alg:solver}).
We define $Q^n = Q^{n,\pi^n}$, $V^n = V^{n,\pi^n}$ for brevity.

\paragraph{Approximators}
On the known states, in outer iteration $n$ and inner iteration $k$, we define the best $Q$-approximator $Q^{n,\star}_k$ and the agent's approximator $\widehat Q^{n}_k$ as
\begin{align}
\text{if} \; s \in \K^n: \;
\begin{cases}
Q^{n,\star}_k(s,a) & = \phi(s,a)^\top w_k^{n,\star} + 2b^n_\phi(s,a) \\
\Qhat^{n}_k(s,a) & =
\phi(s,a)^\top \widehat w_k^n + b^n_\phi(s,a)
\end{cases}
\end{align}
Otherwise, we set them to be the same as $b^n(s,a)$.
We omit either $n$ or $k$ when there is no possibility of confusion.

\newpage
\section{Main Analysis}
We start our analysis by showing some auxiliary lemmas which we will later use to prove \cref{prop:MainAnalysis}. In particular, \cref{lem:PartialOptimism} and \cref{lem:Aneg} are variations of the corresponding lemmas in \citep{agarwal2020pc}.

We start by recalling the performance difference lemma (e.g., \citep{kakade2002approximately}) which states that for any two policies $\pi,\pi'$ we can write
\begin{align}
	(V^{\pi'} - V^{\pi})(s_0) = \frac{1}{1-\gamma}\E_{(s,a) \sim \pi} A^{\pi'}(s,a)
\end{align}
where $A^\pi$ is the advantage function associated with $\pi$.

The following lemma is similar to lemma B.2 \citep{agarwal2020pc}.
\begin{lemma}[Partial optimism]
\label[lemma]{lem:PartialOptimism}
Fix a policy $\pitilde$ that never takes $a^\dagger$. Define the policy $\pitilde^n$ on $\M^n$ such that $\pitilde^n(\cdot \mid s) = \pitilde(\cdot \mid s)$ if $s \in \K^n$ and $\pitilde^n(a^\dagger \mid s) = 1$ if $s \not \in \K^n$.
In any episode $n$ it holds that
\begin{align}
V^{\pitilde}(s_0) + \frac{1}{1-\gamma}\E_{s\sim\pitilde \mid s_0} 2b^{n}_\phi(s,\pitilde) \leq V^{n,\pitilde^n}(s_0).\end{align}
\end{lemma}
\begin{proof}
Notice that $\pitilde^n$ always takes an action where $b^n_\1(s,a) = 0$. A quick computation gives:
\begin{equation}
\label{eqn:Vmax}
V^{n,\pitilde^n}(s) \leq \frac{3}{1-\gamma}.
\end{equation}
and in particular, if $s\not\in\K^n$ then $V^{n,\pitilde^n}(s) = \frac{3}{1-\gamma}$ as the policy self-loops in $s$ by taking $a^\dagger$ there.
Using the performance difference lemma we get:
\begin{align}
& (1-\gamma) \left( V^{n,\pitilde^n}(s_0) -  V^{n,\pitilde}(s_0) \right)= \\
& = \E_{(s,a) \sim \pitilde \mid s_0} \Big[Q^{n,\pitilde^n}(s,\pitilde^n) - Q^{n,\pitilde^n}(s,\pitilde) \Big] \\
& =  \E_{(s,a) \sim \pitilde \mid s_0} \Big[ \left(Q^{n,\pitilde^n}(s,\pitilde^n) - Q^{n,\pitilde^n}(s,\pitilde) \right) \unkn \Big] \\
& =  \E_{(s,a) \sim \pitilde \mid s_0} \Big[ \left(\frac{3}{1-\gamma} - Q^{n,\pitilde^n}(s,\pitilde) \right) \unkn \Big]\\
& =  \E_{(s,a) \sim \pitilde \mid s_0} \Big[ \left( \frac{3}{1-\gamma} - \underbrace{r(s,\pitilde)}_{\leq 1} - \underbrace{2b^n_\phi(s,\pitilde)}_{\leq 2} - b^n_\1(s,\pitilde) - \underbrace{\gamma \E_{s' \sim p(s,\pitilde)}V^{n,\pitilde^n}(s')}_{\leq \frac{3\gamma}{1-\gamma} \; \text{by \cref{eqn:Vmax}}} \right)\unkn  \Big].
\end{align}
We have $ \underbrace{r(s,\pitilde)}_{\leq 1} + \underbrace{2b^n_\phi(s,\pitilde)}_{\leq 2} + \underbrace{\gamma \E_{s' \sim p(s,\pitilde)}V^{n,\pitilde^n}(s')}_{\leq \frac{3\gamma}{1-\gamma} \; \text{by \cref{eqn:Vmax}}}  \leq 3 + \frac{3\gamma}{1-\gamma} \leq \frac{3}{1-\gamma}$. Continuing the chain above:
\begin{align}
& \geq  \E_{(s,a) \sim \pitilde \mid s_0} \Big[- b^n_\1(s,\pitilde) \unkn \Big] \\
& = \E_{(s,a) \sim \pitilde \mid s_0} \Big[- b^n_\1(s,\pitilde) \Big].
\end{align}
Thus,
\begin{align}
	V^{n,\pitilde^n}(s_0) &\geq V^{n,\pitilde}(s_0) - \frac{1}{1-\gamma}\E_{(s,a) \sim \pitilde \mid s_0}  b^n_\1(s,a)\\
				&= V^{\pitilde}(s_0) +\frac{1}{1-\gamma}\E_{(s,a)\sim \pitilde \mid s_0}b^n(s,a) - \frac{1}{1-\gamma}\E_{(s,a) \sim \pitilde \mid s_0}  b^n_\1(s,a)\\
				&= V^{\pitilde}(s_0) + \frac{1}{1-\gamma}\E_{(s,a)\sim \pitilde \mid s_0}2b^n_\phi(s,a)
\end{align}
\end{proof}

The following lemma is similar to lemma A.1 in \citep{agarwal2020pc}.
\begin{lemma}[Negative Advantage]
\label[lemma]{lem:Aneg}
We have $$A^{n}(s,\pitilde^n)\unkn \leq 0.$$
\end{lemma}
\begin{proof}
Assume $s \not \in \K^n$. In such state, $\pitilde^n$ takes action $a^\dagger$ and self-loops in $s$, where a reward $= 3$ is received for the first timestep. Thus, for $Q^n = Q^{n,\pi^n}$,
$$Q^n(s,\pitilde^n) = 3 + \gamma V^{n}(s).$$
In addition, in $s \not \in \K^n$ an action $a \neq a^\dagger$ such that $b^n_\1(s,a) = \frac{3}{1-\gamma}$ must exist. In such case, $\pi^n$ always takes one such action; this is because $\pi^n$ by definition is a mixture of the policies $\pi_1,\dots,\pi_{K-1}$ computed by \cref{alg:solver}, and they all choose an action with the indicator bonus if the state $s \notin \K^n$, see line \ref{line:policydef} in \cref{alg:solver}. Therefore
$$V^{n}(s) \geq \frac{3}{1-\gamma}.$$
Combining the two expressions we obtain that, in any state $s \not \in \K^n$,
$$
A^{n}(s,\pitilde^n) = Q^{n}(s,\pitilde^n) - V^{n}(s) = \Big[ 3 + \gamma V^{n}(s) - V^{n}(s) \Big] = 3 - (1-\gamma) V^{n}(s) \leq 0.
$$
\end{proof}

\begin{proposition}[Analysis of \Alg]
\label[proposition]{prop:MainAnalysis}
With probability at least $1-\delta$ it holds that
\begin{align}
\frac{1}{N}\sum_{n=1}^N\( V^{\pitilde} - V^{\pi^n} \)(s_0) \leq \frac{\mathcal R(K)}{(1-\gamma)K} + \frac{2\sqrt{2A\mathcal E_n}}{1-\gamma} +   \frac{1}{{\sqrt{N}}} \times \widetilde O\(\frac{\sqrt{\beta d}}{(1-\gamma)^2}\)
\end{align}
\end{proposition}
\begin{proof}
Fix a policy $\pitilde$ on $\M$ ($\pitilde$ does not take $a^\dagger$ since $a^\dagger$ is not available on $\M$).
Consider the following decomposition for an outer episode $n$ (recall the policy $\pi^n$ is the mixture policy of the policies $\pi_0,\dots,\pi_{K-1}$ computed by the \textsc{Solver}, see \cref{sec:Notation} for more details)
\begin{align*}
\numberthis{\label{eqn:MainStart}}
\( V^{\pitilde} - V^{\pi^n} \)(s_0) & =  \underbrace{V^{\pitilde}(s_0) +  \frac{1}{1-\gamma} \E_{s\sim\pitilde \mid s_0}2b^{n}_\phi(s,\pitilde)}_{ \leq V^{n,\pitilde^n}(s_0) \; \text{by \cref{lem:PartialOptimism}}} \underbrace{- V^{\pi^n}(s_0) - \frac{1}{1-\gamma} \E_{s\sim\pi^n \mid s_0} b^n(s,\pi^n)}_{\defeq - V^n(s_0)} \\
& + \frac{1}{1-\gamma}  \underbrace{\Big[ -\E_{s\sim\pitilde \mid s_0} 2b^{n}_\phi(s,\pitilde) + \E_{s\sim\pi^n \mid s_0} b^n(s,\pi^n)\Big]}_{\defeq B^n}
\end{align*}
We put the term involving $B^n$ aside for a moment and use the performance difference lemma to obtain
\begin{align*}
\numberthis{\label{eqn:this}}
V^{n,\pitilde^n}(s_0)  - V^n(s_0)
& = \frac{1}{1-\gamma}\E_{s \sim \pitilde^n \mid s_0} \Big[ \underbrace{Q^n(s,\pitilde^n) - V^n(s)}_{A^{n}(s,\pitilde^n)}\Big]  \\
& = \frac{1}{1-\gamma}\E_{s \sim \pitilde^n \mid s_0} \Big[A^n(s,\pitilde^n)\kn + \underbrace{A^n(s,\pitilde^n)\unkn}_{\leq 0 \text{ by \cref{lem:Aneg}}} \Big]\\
& = \frac{1}{1-\gamma}\E_{s \sim \pitilde^n \mid s_0} \Big[A^n(s,\pitilde)\kn \Big]
\end{align*}
where the last step is because on states $s \in \K^n$ we have $\pitilde^n(\cdot \mid s) = \pitilde(\cdot \mid s)$; using this, we can derive
\begin{align*} \numberthis{\label{eqn:advantage decomposition}}
& = \frac{1}{1-\gamma} \Bigg[ \E_{s \sim \pitilde^n \mid s_0} \Ahat^n(s,\pitilde)\kn + \E_{s \sim \pitilde^n \mid s_0} \Big[ A^n(s,\pitilde) -  \Ahat^n(s,\pitilde)\Big]\kn \Bigg] \\
& \leq \frac{1}{1-\gamma} \Bigg[ \underbrace{\sup_{s\in\K^n} \Ahat^n(s,\pitilde)}_{\text{term 1}} + \underbrace{\E_{s \sim \pitilde^n \mid s_0} \Big[ A^n(s,\pitilde) -  A^{n,\star}(s,\pitilde)\Big]\kn}_{\text{term 2}} \\&\qquad\qquad+ \underbrace{\E_{s \sim \pitilde^n \mid s_0} \Big[ A^{n,\star}(s,\pitilde) -  \Ahat^n(s,\pitilde)\Big]\kn}_{\text{term 3}} \Bigg].
\end{align*}

The second term is the approximation error, and we can bound it as follows by taking absolute values and using \fullref{lem:DistributionDominance}

\begin{align*}
\numberthis{\label{eqn:thatA}}
	\E_{s \sim \pitilde^n \mid s_0} \Big[ A^n(s,\pitilde) -  A^{n,\star}(s,\pitilde)\Big]\kn
	\leq & \E_{s \sim \pitilde^n \mid s_0} \Big| A^n(s,\pitilde) -  A^{n,\star}(s,\pitilde)\Big|\kn \\
	\leq & \E_{s \sim \pitilde \mid s_0} \Big| A^n(s,\pitilde) -  A^{n,\star}(s,\pitilde)\Big|\kn\\
	= & \E_{s \sim \pitilde \mid s_0} \Big| \frac{1}{K}\sum_{k=0}^{K-1} A_k^n(s,\pitilde) -  A_k^{n,\star}(s,\pitilde)\Big|\kn \\
\end{align*}

Now we focus on the third term in \eqref{eqn:advantage decomposition};  \fullref{lem:ValidityOfConfidenceItervals} ensures that with probability at least $1-\frac{\delta}{2}$ it holds that
\begin{align}
\label{eqn:thisQ}
\forall n \in [N], \; \forall k \in \{0,\dots,K-1 \}, \; \forall (s,a)\in\K^n: \quad 0 \leq Q^{n,\star}_k(s,a)-\Qhat^n_k(s,a) \leq 2b^{n}_\phi(s,a).
\end{align}
In what follows we omit the subscript $k$  as we need the bound to hold only for the mixture policy $\pi^n$.
Then $\forall n \in [N],\; \forall (s,a)\in\K^n$:
\begin{align}
  A^{n,\star}(s,a)-\Ahat^n(s,a) & = \frac{1}{K}\sum_{k=0}^{K-1} \Bigg[\(Q^{n,\star}_k(s,a)-\Qhat^n_k(s,a)\) - \underbrace{\(Q^{n,\star}_k(s,\pi^n_k)-\Qhat^n_k(s,\pi^n_k) \)}_{\leq 0} \Bigg]\\
  & \leq Q^{n,\star}(s,a)-\Qhat^n(s,a).
\end{align}
The right hand side is by definition positive using \cref{eqn:thisQ}.
Thus \fullref{lem:DistributionDominance} can be applied to obtain
\begin{align}
\label{eqn:thatQ}
	\E_{s \sim \pitilde^n \mid s_0} \Big[ A^{n,\star}(s,\pitilde) -  \Ahat^n(s,\pitilde)\Big]\kn & \leq \E_{s \sim \pitilde^n \mid s_0} \Big[ Q^{n,\star}(s,\pitilde) -  \Qhat^n(s,\pitilde)\Big]\kn \\
	& \leq \E_{s \sim \pitilde \mid s_0} \Big[ Q^{n,\star}(s,\pitilde) -  \Qhat^n(s,\pitilde)\Big]\kn.
	\end{align}

Together, plugging \cref{eqn:thatQ,eqn:thatA} back into \cref{eqn:advantage decomposition}, \cref{eqn:this} and finally \cref{eqn:MainStart} gives
\begin{align*}
\numberthis{\label{eqn:RefInMain}}
	\( V^{\pitilde} - V^{\pi^n} \)(s_0) & \leq   \frac{1}{1-\gamma} \Bigg[ \underbrace{\sup_{s\in\K^n} \Ahat^n(s,\pitilde)}_{\text{term 1}} + \underbrace{\E_{s \sim \pitilde \mid s_0} \Big| A^n(s,\pitilde) -  A^{n,\star}(s,\pitilde)\Big|\kn}_{\text{term 2}} \\
	& + \underbrace{\E_{s \sim \pitilde \mid s_0} \Big[ Q^{n,\star}(s,\pitilde) -  \Qhat^n(s,\pitilde)\Big]\kn}_{\text{term 3}} + B^n\Bigg]
\end{align*}

We can bound term 1 using \fullref{lem:NPG}.
We then obtain (the online regret $\mathcal R(K)$ is defined in the lemma)
\begin{align}
\sup_{s\in \StateSpace} \Ahat^n(s,\pitilde)\kn = \sup_{s\in \StateSpace}
\frac{1}{K}\sum_{k=0}^{K-1} \E_{a \sim \pitilde(\cdot \mid s)} \Ahat^n_k(s,a)\kn \leq \frac{\mathcal R(K)}{K}.
\end{align}

We can bound the second term in the prior display by invoking \fullref{lem:adv}. The third term is finally bounded by \cref{eqn:thisQ}. As a result, the performance difference has an upper bound:
\begin{align}
\( V^{\pitilde} - V^{\pi^n} \)(s_0) & \leq \frac{1}{1-\gamma}\Bigg[ \frac{\mathcal R(K)}{K} +   2\sqrt{2A\mathcal E^n} + \E_{(s,a) \sim \pitilde \mid s_0} 2b^n_\phi(s,a)\kn + B^n \Bigg]\\
& = \frac{1}{1-\gamma}\Bigg[\frac{\mathcal R(K)}{K} +  2\sqrt{2A\mathcal E^n} +  \E_{s\sim\pi^n \mid s_0} b^n(s,\pi^n) \Bigg]
\end{align}
Averaging over the outer rounds $n\in[N]$ and defining $\sqrt{\mathcal E} \defeq \frac{1}{N}\sum_{n=1}^N \sqrt{\mathcal E^n}$ (where $\sqrt{\mathcal E^n}$ itself is an average of the \textsc{Solver}'s errors $\sqrt{\mathcal E^n} \defeq \frac{1}{K}\sum_{k=0}^{K-1} \sqrt{\mathcal E^n_k}$) gives
\begin{align}
\frac{1}{N}\sum_{n=1}^N \( V^{\pitilde} - V^{\pi^n} \)(s_0) \leq \frac{\mathcal R(K)}{(1-\gamma)K} + \frac{2\sqrt{2A\mathcal E}}{1-\gamma} +   \frac{1}{N(1-\gamma)} \sum_{n=1}^N\E_{s\sim\pi^n \mid s_0} b^n(s,\pi^n)
\end{align}
Finally, \fullref{lem:SumOfBonuses} and \fullref{lem:SumOfIndicators} and a union bound conclude.
\end{proof}

The following lemma is similar to lemma C.1 in \citep{agarwal2020pc}.
\begin{lemma}[Advantage Transfer Error Decomposition]
\label[lemma]{lem:adv}
We have
\begin{align}
\E_{s \sim \pitilde \mid s_0}\Big|A^n(s,\pitilde) - \Ahat^{n,\star}(s,\pitilde) \Big|\kn \leq 2\sqrt{2A\mathcal E^n}. 
\end{align}
\end{lemma}
\begin{proof}
We leave the conditioning on the starting state $s_0$ implicit.
Using \fullref{def:TransferError}
\begin{align}
& = \E_{s \sim \pitilde}\E_{a \sim \pitilde(\cdot \mid s)}\Big| \left(A^n(s,a) - A^{n,\star}(s,a)\right)\kn \Big| \\
& \leq \E_{s \sim \pitilde}\E_{a \sim \pitilde(\cdot \mid s)}\Big| \left(Q^n(s,a) - Q^{n,\star}(s,a)\right)\kn \Big| \nonumber \\
&\quad + \E_{s \sim \pitilde}\E_{a \sim \pi^n(\cdot \mid s)}\Big| \left(Q^n(s,a) - Q^{n,\star}(s,a)\right) \kn \Big| \\
& \stackrel{\text{Jensen}}{\leq} \sqrt{\E_{s \sim \pitilde}\E_{a \sim \pitilde(\cdot \mid s)}\Big[ \left( Q^n(s,a) - Q^{n,\star}(s,a)\right)^2 \kn \Big]}\\
&\qquad\qquad\qquad+ \sqrt{ \E_{s \sim \pitilde}\E_{a \sim \pi^n(\cdot \mid s)}\Big[ \left(Q^n(s,a) - Q^{n,\star}(s,a) \right)^2\kn \Big]} \\
& \leq \sqrt{|\ActionSpace|\E_{s \sim \pitilde}\E_{a \sim \text{Unif}|\ActionSpace|}\Big[ \left(Q^n(s,a) - Q^{n,\star}(s,a)\right)^2\Big]} + \sqrt{|\ActionSpace|\E_{s \sim \pitilde}\E_{a \sim \text{Unif}|\ActionSpace|}\Big[ \left(Q^n(s,a) - Q^{n,\star}(s,a)\right)^2\Big]} \\
& \leq 2\sqrt{|\ActionSpace|}\sqrt{2\mathcal L(w^{n,\star},d^{\pitilde} \circ \text{Unif}|\ActionSpace|,Q^n - b^n)} \\
& \leq 2\sqrt{2|\ActionSpace|\mathcal E^n}.
\end{align}
\end{proof}

The following lemma is similar to B.1 in \citep{agarwal2020pc}.
\begin{lemma}[Distribution Dominance]
\label[lemma]{lem:DistributionDominance}
If $f: \StateSpace \rightarrow \R$ is a positive function then we have
$$\E_{s \sim \pitilde^n} f(s)\kn \leq \E_{s \sim \pitilde} f(s)\kn.$$
\end{lemma}
\begin{proof}
Consider the MDP $\M^n$ but with $f(s)\kn$ as the total reward function in $s$ and let $\widetilde Q,\widetilde V$ be the value functions. Recall that the reward is positive and that $\pitilde^n$ circles back to  $s\not\in\K^n$ once such state is reached. Then the performance difference lemma ensures
\begin{align}
	&\E_{s \sim \pitilde^n} f(s)\kn - \E_{s \sim \pitilde} f(s)\kn \\
	& = (1-\gamma)(\widetilde V^{\pitilde^n}  -
	\widetilde V^{\pitilde}) \\
	& =  \E_{s \sim \pitilde} \Big[  \left( \underbrace{\widetilde Q^{\pitilde^n}(s,\pitilde^n)}_{=0}-\widetilde Q^{\pitilde^n}(s,\pitilde) \right) \unkn \Big] \\
	& +   \E_{s \sim \pitilde} \Big[  \left( \underbrace{\widetilde Q^{\pitilde^n}(s,\pitilde^n)}_{= \widetilde Q^{\pitilde^n}(s,\pitilde)}-\widetilde Q^{\pitilde^n}(s,\pitilde) \right) \kn \Big] \\ &
		 \leq 0.
\end{align}
\end{proof}

\newpage
\section{NPG Guarantees}
Consider a fixed episode $n$ where NPG is invoked (we omit the dependence on $n$ in the notation) and notice that the set $\K^n$ is fixed.
\begin{lemma}[NPG lemma]
\label[lemma]{lem:NPG}
Fix $n$. If $K \geq 4\ln |\ActionSpace|$ and the learning rate is $\eta = \frac{\sqrt{\ln |\ActionSpace|}}{\sqrt{K}W}$, then $\eta|\widehat A_k(\cdot,\cdot)| \leq 1$ and we have for any fixed state $s \in \K^n$ and distribution $\widetilde\pi(\cdot \mid s)$
\begin{align}
\sum_{k=0}^{K-1} \E_{a \sim \widetilde\pi(\cdot \mid s)}  \widehat A_k(s,a)\kn \leq 2W\sqrt{\ln |\ActionSpace | K} \defeq \mathcal R(K).
\end{align}
\end{lemma}

\begin{proof}
The update rule in known states reads
\begin{align}
\pi_{k+1}(\cdot \mid s) & \propto \pi_{k}(\cdot \mid s) e^{\eta \widehat Q_k(s,\cdot)} \\
& \propto \pi_{k}(\cdot \mid s) e^{\eta \widehat Q_k(s,\cdot)}e^{-\eta\Vhat_k(s)} \\
& = \pi_{k}(\cdot \mid s) e^{\eta \widehat A_k(s,\cdot)}.
\end{align}
Denote the normalizer $z_k(s) = \sum_{a'}\pi_{k}(a' \mid s)e^{\eta \widehat A_k(s,a')}$. The update rule in known states can be written as
\begin{align}
\pi_{k+1}(\cdot \mid s) & = \frac{\pi_{k}(\cdot \mid s) e^{\eta \widehat A_k(s,\cdot)}}{z_k(s)}.
\end{align}

Then we have the following equality for any state $s \in \K$:
\begin{align*}
&\kl(\widetilde\pi(\cdot \mid s) \mid \mid  \pi_{k+1}(\cdot \mid s)) - \kl(\widetilde\pi(\cdot \mid s) \mid \mid \pi_k(\cdot \mid s)) \\
& = \sum_{a} \widetilde\pi(a \mid s) \ln\frac{ \widetilde\pi(a \mid s)}{\pi_{k+1}(a\mid s)} - \sum_{a} \widetilde\pi(a \mid s) \ln\frac{ \widetilde\pi(a \mid s)}{\pi_k(a\mid s)} \\
& = \sum_{a} \widetilde\pi(a \mid s) \ln\frac{ \pi_k(a \mid s)}{\pi_{k+1}(a\mid s)} \\
&  = \sum_{a} \widetilde\pi(a \mid s) \ln \( z_k e^{-\eta \widehat A_k(s,a)} \)\\
&  = - \eta\sum_{a} \widetilde\pi(a \mid s)  \widehat A_k(s,a) + \ln z_k(s).
\numberthis{\label{eqn:kl}}
\end{align*}
We show that for any know state $s$ we have $\ln z_k(s) \leq \eta^2W^2$. To see this, we use the fact that $|\eta \widehat A_k(\cdot,\cdot)| \leq 1$ which allows us to use the inequality $e^x \leq 1+x+x^2$ to claim for any known state
\begin{align}
\ln z_k(s) & = \ln \( \sum_{a'}\pi_{k}(a' \mid s)e^{\eta \widehat A_k(s,a')} \) \\
& \leq \ln \( \sum_{a'}\pi_{k}(a' \mid s)(1+\eta \widehat A_k(s,a') + \eta^2 \widehat A^2_k(s,a') ) \)  \\
& \leq \ln \(1+ \eta^2 W^2 \) \leq \eta^2W^2.
\end{align}

Plugging the above result into \cref{eqn:kl} and summing over $k$ gives
\begin{align}
&\kl(\widetilde\pi(\cdot \mid s) \mid \mid  \pi_{K}(\cdot \mid s)) - \kl(\widetilde\pi(\cdot \mid s) \mid \mid \pi_0(\cdot \mid s)) \\
&  = \sum_{k=0}^{K-1} \Big[ \kl(\widetilde\pi (\cdot \mid s) \mid \mid  \pi_{k+1}(\cdot \mid s)) - \kl(\widetilde\pi(\cdot \mid s) \mid \mid \pi_k(\cdot \mid s)) \Big] \\
&  \leq - \eta \sum_{k=0}^{K-1} \sum_{a} \widetilde\pi(a \mid s)  \widehat A_k(s,a) + \eta^2W^2K.
\end{align}
Recalling that the $\kl$ divergence is positive and that $\kl(\widetilde\pi(\cdot \mid s) \mid \mid \pi_0(\cdot \mid s)) \leq \ln | \ActionSpace |$ for know states gives:
\begin{align}
\eta \sum_{k=0}^{K-1} \E_{a \sim \widetilde\pi (\cdot \mid s)} \widehat A_k(s,a)\kn \leq \ln |\ActionSpace | + \eta^2W^2K.
\end{align}
Choosing $\eta = \frac{\sqrt{\ln |\ActionSpace|}}{\sqrt{K}W}$ finally gives
\begin{align}
\sum_{k=0}^{K-1} \E_{a \sim \widetilde\pi (\cdot \mid s) }  \widehat A_k(s,a) \kn \leq 2W\sqrt{\ln |\ActionSpace | K} \defeq \mathcal R(K).
\end{align}
\end{proof}

\newpage
\section{Iteration and Sample Complexity}
\label{sec:Complexity}
In this section we examine the sample and iteration complexity of the algorithm
\begin{lemma}[Iteration Complexity]
\label{lem:Complexity}
With probability at least $1-\delta$ we have
\begin{align}
	\frac{1}{N}\sum_{n=1}^N\( V^{\pitilde} - V^{\pi^n} \)(s_0)  \leq  \epsilon + \frac{2\sqrt{2A\mathcal E}}{1-\gamma}
\end{align}
with the number of inner iterations $K$ and the number of outer iterations $N$ no larger than
\begin{align}
	K  =  \widetilde O\(\frac{\ln|\mathcal A|W^2}{(1-\gamma)^2\epsilon^2}\),
	\qquad
	N  = \widetilde O\(\frac{d^2}{(1-\gamma)^8\epsilon^2}\).
\end{align}
\end{lemma}
\begin{proof}
Consider \fullref{prop:MainAnalysis}.
We need ensure
\begin{align}
\frac{\mathcal R(K)}{(1-\gamma)K} = \frac{2W}{(1-\gamma)}\sqrt{\frac{\ln |\ActionSpace | }{K}}\leq \frac{\epsilon}{2} \quad \longrightarrow \quad K = \widetilde O\(\frac{\ln|\mathcal A|W^2}{(1-\gamma)^2\epsilon^2}\).
\end{align}
This gives the inner iteration complexity. Next ($\beta$ comes from \cref{eqn:beta})
\begin{align}
\frac{1}{{\sqrt{N}}} \times \widetilde O\(\frac{\sqrt{\beta d}}{(1-\gamma)^2}\) \leq \frac{\epsilon}{2} \quad \longrightarrow \quad N & = \widetilde O\(\frac{d\beta}{(1-\gamma)^4\epsilon^2}\) \\
& =\widetilde O\(\frac{d}{(1-\gamma)^4\epsilon^2}\)  \times \widetilde O\(\frac{d}{(1-\gamma)^4}\) \\
& = \widetilde O\(\frac{d^2}{(1-\gamma)^8\epsilon^2}\)
\end{align}
gives the outer iteration complexity.
\end{proof}

\begin{lemma}[Sample Complexity]
In the same setting as \fullref{lem:Complexity}, the total number of sampled trajectories is
\begin{align}
\widetilde O\(\frac{d^3}{(1-\gamma)^{12}\epsilon^3}\)
\end{align}
or equivalently
\begin{align}
\widetilde O\(\frac{d^3}{(1-\gamma)^{13}\epsilon^3}\)
\end{align}
samples.
\end{lemma}

\begin{proof}
Every time the bonus switches, \cref{alg:solver} is invoked, and runs for $K$ iterations. From  \fullref{lem:UnionBound} we know that once data are collected, they can be reused for the next $\kappa$ policies (defined in \cref{eqn:kappa}). Let $S$ be the number of bonus switches given in \fullref{lem:Switches};  then fresh data is collected a total of (for the definitions of the symbols, please see \cref{tab:MainNotation})
\begin{align}
S \times \ceil{\frac{K}{\kappa}}
& = \widetilde O\(d \times \frac{2\ln(1/\delta)(\frac{\sqrt{\ln |\ActionSpace|}}{\sqrt{K}W})\( B + W \)}{(1-\gamma)\ln 2} \times K \) \\
& = \widetilde O \( d \times  \(\frac{B}{W} + 1\)\frac{\sqrt{K}}{1-\gamma} \) \\
& = \widetilde O\(d \times \frac{1}{1-\gamma}  \times \frac{W}{(1-\gamma)\epsilon}\)\\
& = O\(\frac{d}{(1-\gamma)^4\epsilon} \)
\end{align}
times (as $W \geq B$). Every time data is collected by the critic at most $N$ rollouts are performed, giving the total number of trajectories:
\begin{align}
	N \times \widetilde O\(\frac{d}{(1-\gamma)^4\epsilon} \) = \widetilde O\(\frac{d^3}{(1-\gamma)^	{12}\epsilon^3}\).
\end{align}
The sample complexity is then obtained by multiplying the above result by $t_{max}$ in \cref{tab:MainNotation}, which is a uniform bound on the trajectory length in the event we consider.
\end{proof}

\newpage
\section{Regression with Monte Carlo and Importance Sampling}
\label{sec:Critic}
In this section we derive high probability confidence intervals for Monte Carlo with importance sampling.
\begin{itemize}
	\item \fullref{sec:ISE} gives generic properties of the importance sampling estimator when the target and behavioral policies are not too different.
	\item \fullref{sec:Perturbations} examines the effect of small perturbations to policies; this is needed in the union bound in the section described below.
	\item \fullref{sec:Regression} gives the actual confidence intervals for the $Q$-values for the algorithm we examine in this paper
\end{itemize}

\subsection{Importance Sampling Estimator}
The importance sampling ratio used in this work starts from the timestep $t=2$: since we are estimating the $Q$-values of policies, the first state-action from the cover is always fixed, the two policies in the ratio at $t=1$ cancel each other out.
\label{sec:ISE}
\begin{definition}[Importance Sampling Estimator]
\label{def:IS} Let $t$ be a positive discrete random variable with probability mass function $\Pro(t=\tau)= \gamma^{\tau-1}(1-\gamma)$, and let $\{(s_\tau,a_\tau,r_\tau)\}_{\tau=1,\dots,t}$ be a random trajectory of length $t$ obtained by following a fixed ``behavioral'' policy $\underline \pi$ from $(s,a)$. The importance sampling estimator of the target policy $\pi$ is:
\begin{align}
\label{eqn:IS}
\( \Pi_{\tau=2}^t \frac{\pi(s_\tau,a_\tau)}{\underline \pi(s_\tau,a_\tau)}\) \frac{r_t}{1-\gamma}.
\end{align}
\end{definition}

In this section, we will focus on a specific type of behavior and target policies, which are related as
\begin{align}
\label{eqn:policy-conditions}
\forall (s,a), \quad  \quad \pi(a \mid s) = \underline \pi(a \mid s) \times \frac{e^{c(s,a)}}{\sum_{a'}\underline \pi(a' \mid s) e^{c(s,a')}}, \quad \quad  \sup_{(s,a)} |c(s,a)| \leq \frac{(1-\gamma)\ln \(1 + \epsilon'\)}{2 \ln(1/\delta')}.
\end{align}

For such policies, we have the following results.

\begin{lemma}[Policy Ratio]
\label[lemma]{lem:PolicyRatio}
Assume that $\pi,\underline \pi$  are related through \cref{eqn:policy-conditions} and that $c \defeq \sup_{(s,a)}|c(s,a)|$.
Then
\begin{align}
e^{-2c} \leq \sup_{(s,a)}\frac{\pi(a \mid s)}{\underline \pi(a \mid s)} \leq e^{2c}.\end{align}
\end{lemma}
\begin{proof}
The following chain of inequalities is true.
\begin{align}
e^{-2c}\leq \frac{e^{-c}}{\sum_{a'}\underline \pi(a' \mid s) e^{c}} \leq \frac{\pi(a \mid s)}{\underline \pi(a \mid s)} =  \frac{e^{c(s,a)}}{\sum_{a'}\underline \pi(a' \mid s) e^{c(s,a')}} \leq \frac{e^{c}}{\sum_{a'}\underline \pi(a' \mid s) e^{-c}} = e^{2c}.
\end{align}
\end{proof}
In addition, for the policies in \cref{eqn:policy-conditions} we can also examine the bias and variance of the importance sampling estimator.
\begin{lemma}[Bias and Variance of Importance Sampling Estimator]
\label[lemma]{lem:bias-variance}
Let $\underline \pi$ be a fixed  behavioral policy. If $\pi$ is a fixed target policy with the same support as $\underline \pi$ then \cref{eqn:IS} is an unbiased estimator of the value of $\pi$ from $(s,a)$. In addition, assume $\pi,\underline \pi$ are related by \cref{eqn:policy-conditions} where in particular $c(s,a)$ satisfies the constraint in \cref{eqn:policy-conditions}. Let $R_{max}$ be a deterministic upper bound to the maximum absolute value of the reward $r_t$.
Then with probability at least $1-\delta'$ the importance sampling estimator in \cref{eqn:IS} is bounded in absolute value by $\frac{1+\epsilon'}{1-\gamma}R_{max}$ and the random timestep $t$ in the importance sampling estimator is bounded by $\frac{\ln1/\delta'}{1-\gamma}$.

\end{lemma}

\begin{proof}
It is well known that the importance sampling estimator is unbiased \citep{precup2000eligibility}.
For the high probability bound we proceed as follows.
Using \fullref{lem:PolicyRatio} we claim
\begin{align}
\( \sup_{(s,a)}\frac{\pi(a \mid s)}{\underline \pi(a \mid s)} \)^{t-1} & \leq e^{2(t-1)c}.
\end{align}
We show that $t$ is small with high probability:
\begin{align}
\Pro(t> \tau) & = \sum_{t=\tau+1}^\infty\gamma^{\tau-1}(1-\gamma)\\
& = \gamma^\tau \sum_{t=0}^\infty\gamma^{\tau}(1-\gamma) \\
& = \gamma^\tau \defeq \delta'.
\end{align}
This implies
\begin{align}
\tau = \frac{\ln \delta'}{\ln \gamma} =  \frac{\ln 1/\delta'}{\ln 1/\gamma} \leq \frac{\ln 1/\delta'}{1-\gamma}
\end{align}
In the complement of the above event:
\begin{align}
\( \sup_{(s,a)}\frac{\pi(a \mid s)}{\underline \pi(a \mid s)} \)^{t-1} \leq  e^{2(\tau-1)c}.
\end{align}
We require that the exponential above be $\leq 1 + \epsilon'$, leading to the condition:
\begin{align}
2(\tau-1)c \leq \ln\(1 + \epsilon'\) \Rightarrow c \leq \frac{\ln \(1 + \epsilon'\)}{2(\tau-1)}.
\end{align}
Under the assumption of \cref{eqn:policy-conditions}, the above condition holds because
\begin{align}
	\sup_{(s,a)} |c(s,a)| = c \leq \frac{(1-\gamma)\ln \(1 + \epsilon'\)}{2 \ln(1/\delta')} \leq \frac{\ln \(1 + \epsilon'\)}{2\tau}
\leq \frac{\ln \(1 + \epsilon'\)}{2(\tau-1)},
\end{align}
Therefore, we can ensure 
\begin{align}
\Pro\( \Big\{ \forall \pi \; \text{satisfying } \cref{eqn:policy-conditions}, \( \sup_{(s,a)}\frac{\pi(s,a)}{\underline \pi(s,a)} \)^{t-1} \leq \(1 + \epsilon'\) \Big\} \bigcap  \Big\{ t \leq \frac{\ln 1/\delta'}{1-\gamma} \Big\}\) \geq 1-\delta'.
\end{align}
Then with probability at least $1-\delta'$ if the importance sampling ratio is upper bounded
\begin{align}
\Pi_{\tau=2}^t \frac{\pi(s_\tau,a_\tau)}{\underline \pi(s_\tau,a_\tau)} \leq \(  \sup_{(s,a)}\frac{\pi(a \mid s)}{\underline \pi(a \mid s)} \)^{t-1} \leq 1+\epsilon'
\end{align}
the thesis follows.
\end{proof}

\subsection{Small Perturbations to Policies}
\label{sec:Perturbations}
In this section we examine the effect on the loss of small perturbations to the algorithm policies. This is useful when dealing with a discretization argument in \fullref{sec:Regression}. We highlight that the $\epsilon''$ in this section concerns the discretization error in the union bound in \fullref{sec:Regression}, and is not to be confused with the value that $\epsilon'$ takes in \fullref{sec:ISE} (in particular, $\epsilon'$ is implicitly defined in \cref{eqn:kappa}).
\begin{lemma}[Difference and Ratio of Nearby Policies]
\label[lemma]{lem:NearbyPolicies}
Fix $\underline \pi$ and assume that $\pi,\pi',\epsilon''$ satisfy $\forall (s,a)$ the following conditions for some function $b(s,a)$:
\begin{align*}
& \| \phi(s,a)\|_2 \leq 1, \quad \|w-w' \|_2 \defeq \epsilon'' \leq 1 \\  \pi'(a\mid s) & \defeq \underline \pi(a \mid s) \times \frac{e^{c'(s,a)}}{\sum_{a'}\underline \pi( a' \mid s) e^{c'(s,a')}}, \quad \quad \text{where} \quad  c'(s,a) \defeq b(s,a) + \phi(s,a)^\top  w' \\
\pi(a\mid s) & \defeq \underline \pi(a \mid s) \times \frac{e^{c(s,a)}}{\sum_{a'}\underline \pi( a' \mid s) e^{c(s,a')}}, \quad \quad \text{where} \quad  c(s,a) \defeq b(s,a) + \phi(s,a)^\top  w.
\numberthis{\label{eqn:pi}}
\end{align*}

If $\underline \pi(a \mid s) = 0$ then $\pi(a \mid s) = \pi'(a \mid s) = 0$. Otherwise we have the following inequalities:
\begin{align}
\frac{\pi'(a\mid s)}{\pi(a\mid s)} \leq 1+4\epsilon'', \quad \quad \frac{\pi(a\mid s)}{\pi'(a\mid s)} \leq 1+4\epsilon'', \quad \quad
\sum_{a} | \pi'(a\mid s) - \pi(a \mid s) | \leq 8\epsilon''.
\end{align}
\end{lemma}
\begin{proof}
Dividing the two expressions gives
\begin{align}
\frac{\pi'(a\mid s)}{\pi(a\mid s)} & = \frac{e^{c'(s,a)}}{e^{c(s,a)}} \times \frac{\sum_{a'}\underline \pi( a' \mid s) e^{c(s,a')}}{\sum_{a'}\underline \pi( a' \mid s) e^{c'(s,a')}} \\
& = e^{\phi(s,a)^\top (w'-w)} \times \sum_{a''} \underline \pi(a'' \mid s) \frac{e^{b(s,a'') + \phi(s,a'')^\top w'}e^{\phi(s,a'')^\top (w-w')}}{\sum_{a'}\underline \pi( a' \mid s) e^{c'(s,a')}} \\
& = e^{\phi(s,a)^\top (w'-w)} \times  \sum_{a''} \pi'(a'' \mid s) e^{\phi(s,a'')^\top (w-w')} \\
& \leq e^{2\|w - w'\|_2} \leq 1+4\|w-w' \|_2.
\end{align}

The last step follows if $\|w - w' \|_2 \leq 1$ by the inequality $e^x \leq 1+2x$ if $x \in [0,1]$. By symmetry, we obtain the other inequality concerning the ratio of the policies in the statement of the lemma. Using the expression derived above we can write (the second expression below follows by symmetry)
\begin{align}
\pi'(a\mid s) - \pi(a \mid s) & \leq 4\|w-w'\|_2 \pi(a\mid s) \\
\pi(a\mid s) - \pi'(a \mid s) & \leq 4\|w-w'\|_2 \pi'(a\mid s)
\end{align}
Taking absolute values and summing over the actions leads to the third expression in the lemma's statement.
\end{proof}

\begin{lemma}[Stability of the $Q$-values]
\label[lemma]{lem:StabilityQ}
Let $\pi',\pi,\epsilon''$ as in \cref{eqn:pi}. It holds that
\begin{align}
\forall (s,a), \quad \quad \quad |\(Q_k^{\pi'} - Q_k^{\pi} \)(s,a)|  	& \leq \frac{8\epsilon''}{1-\gamma} \times \sup_{(s'',a''),\pi''\in\{\pi,\pi'\}} |Q_k^{\pi''}(s'',a'')|.
\end{align}
\end{lemma}
\begin{proof}
Using the performance difference lemma and \fullref{lem:NearbyPolicies} we can write
\begin{align}
\(Q_k^{\pi'} - Q_k^{\pi} \)(s,a) & = \sum_{t=2}^\infty \gamma^{t-1} \E_{s_t \sim \pi \mid (s,a)} \Big[ \sum_{a_t}\pi'(a_t \mid s_t)Q_k^{\pi'}(s_t,a_t) - \sum_{a_t}\pi(a_t \mid s_t) Q_k^{\pi'}(s_t,a_t) \Big] \\
& \leq \frac{1}{1-\gamma} \times \sup_s \sum_{a} |\pi'(a \mid s) - \pi(a \mid s)| \times \sup_{(s'',a'')} |Q_k^{\pi'}(s'',a'')| \\
& \leq \frac{8\epsilon''}{1-\gamma}\sup_{(s'',a'')} |Q_k^{\pi'}(s'',a'')|.
\end{align}
Symmetry concludes.
\end{proof}

\begin{lemma}[Stability of the Empirical $Q$-values]
\label[lemma]{lem:StabilityEmpiricalQ}
Let $\pi',\pi,\epsilon''$ as in \cref{eqn:pi}. For any trajectory $\{s_1,a_1,r_1,\dots,s_t,a_t,r_t\}$ of length $t\leq t_{max}$ if $4\epsilon''t_{max} \leq 1$ it holds that
\begin{align}
\Bigg| \( \Pi_{\tau=2}^t \frac{\pi'(s_\tau,a_\tau)}{\underline \pi(s_\tau,a_\tau)} \) \frac{r_t}{1-\gamma} - \(\Pi_{\tau=2}^t \frac{\pi(s_\tau,a_\tau)}{\underline \pi(s_\tau,a_\tau)} \)\frac{r_t}{1-\gamma} \Bigg| \leq 8\epsilon''t_{max} \times  \max_{\pi'' \in \{\pi,\pi'\}} \(\Pi_{\tau=2}^t \frac{\pi''(s_\tau,a_\tau)}{\underline \pi(s_\tau,a_\tau)} \) \frac{r_t}{1-\gamma}
\end{align}
\end{lemma}
\begin{proof}
Using \fullref{lem:NearbyPolicies} we can write
\begin{align}
& \( \Pi_{\tau=2}^t \frac{\pi'(s_\tau,a_\tau)}{\underline \pi(s_\tau,a_\tau)} - \Pi_{\tau=2}^t \frac{\pi(s_\tau,a_\tau)}{\underline \pi(s_\tau,a_\tau)} \) \frac{r_t}{1-\gamma} \\
& = \Pi_{\tau=2}^t \frac{\pi(s_\tau,a_\tau)}{\underline \pi(s_\tau,a_\tau)} \( \Pi_{\tau=2}^t \frac{\pi'(s_\tau,a_\tau)}{\pi(s_\tau,a_\tau)} - 1 \) \frac{r_t}{1-\gamma}.
\end{align}
Now, apply \fullref{lem:NearbyPolicies} and the condition $t \leq t_{max}$ to the middle term to derive
\begin{align}
	\Pi_{\tau=2}^t \frac{\pi'(s_\tau,a_\tau)}{\pi(s_\tau,a_\tau)} - 1  \leq (1+4\epsilon'')^{t_{max}} - 1.
\end{align}
Recall $t_{max} \geq 1$; for $x \in \R$, when $xt_{max} \leq 1$ but $x \geq 0$, we have the following inequalities:
\begin{align}
	1+x \leq e^{x} \rightarrow (1+x)^{t_{max}} \leq e^{t_{max}x} \leq 1+2t_{max}x.
\end{align}
Let $x = 4\epsilon''$; if $4\epsilon'' t_{max} \leq 1$ then we have
\begin{align}
(1+4\epsilon'')^{t_{max}} - 1 \leq 8\epsilon''t_{max}.
\end{align}
Symmetry concludes.
\end{proof}

\subsection{Regression Guarantees with Importance Sampling}
\label{sec:Regression}
In this section we examine the rate of convergence of the linear regression that uses the importance sampling estimator in the way it is implemented in the algorithm. The bonus $b^n$ and the `reference' expected and empirical covariance matrices $\Sigma^n, \widehat \Sigma^n$ are fixed throughout this section (as they are fixed in all inner iterations of the algorithm). As the outer iteration index $n$ is constant, we often omit it for brevity.

\textbf{Remark}: for additional notation please see \cref{tab:MainNotation}.

\textbf{Remark}: in \cref{eqn:Gdef}, if the trajectory is of length $1$ then the bonus is not added to accommodate the linear MDP framework, and instead it is (half) added directly to the predictor $\Qhat$, resulting in the pessimistic biased estimate described in the main text.

\begin{lemma}[Statistical Rate for a Fixed Target Policy Regression with Importance Sampling]
\label[lemma]{lem:ISrates}
Fix a behavioral policy $\underline \pi$ and a target policy $\pi$ satisfying \cref{eqn:policy-conditions} with $\epsilon' = 1$. Fix a bonus function  $0 \leq b(\cdot,\cdot) \leq B$.
Consider drawing $n$ samples, as follows. For every sample $i\in[n]$ first draw a timestep $t_i \geq 1$ with probability $\Pro(t_i = \tau) = \gamma^{\tau-1}(1-\gamma)$ and a starting state-action $(s_{i1},a_{i1})\sim \rho$ for some distribution $\rho$. Second, draw a trajectory $\{s_{i1},a_{i1},r_{i1},\dots,s_{it_i},a_{it_i},r_{it_i} \}$ from $(s_{i1},a_{i1})$ by following $\underline \pi$ for $t_i-1$ timesteps. Define the random return $G_i$ as \begin{align}
\label{eqn:Gdef}
G_i =
\begin{cases}
\frac{1}{1-\gamma}\Big[r_{it_i} + b(s_{it_i},a_{it_i})\Big] \quad & \text{if} \; t_i \geq 2 \\
\frac{1}{1-\gamma}\Big[r_{it_i}\Big]  \quad & \text{if} \; t_i = 1
\end{cases}
\end{align}
Let $G_{max}$ be a deterministic upper bound to any realization of $G_i$ above (its value is defined in \cref{tab:MainNotation}).
Define the empirical loss and the empirical minimizer
\begin{align}
\label{eqn:Lhat}
\widehat {\mathcal L}(w,\pi) = \frac{1}{2n}\sum_{i=1}^n  \( \phi(s_{i1},a_{i1})^\top w - \Pi_{\tau=2}^{t_i} \frac{\pi(s_{i\tau},a_{i\tau})}{\underline \pi(s_{i\tau},a_{i\tau})} G_i\)^2, \quad \quad \quad \widehat w = \argmin_{\| w \|_2 \leq W} \widehat {\mathcal L}(w,\pi).
\end{align}
Define the true minimizer of the loss in \cref{eqn:L}
\begin{align}
\label{eqn:Largmin}
w^\star \defeq \argmin_{\|w\|_2 \leq W} \mathcal L(w,\rho,Q^{b,\pi} - b)
\end{align}
where $Q^{b,\pi}$ is the state-action value function of $\pi$ on $\M(\StateSpace,\ActionSpace,p,r+b,\gamma)$.
Let $\Sigma = n \E_{(s,a)\sim \rho}\phi(s,a)\phi(s,a)^\top + \lambda I$. With probability at least $1-(n+1)\delta'$
\begin{align}
\|w^\star - \widehat w \|^2_{\Sigma} \leq 2(C_1+C_2\ln\frac{1}{\delta'}) + 2\lambda W^2 = \widetilde O(d W^2).
\end{align}

\end{lemma}
\begin{proof} The hypotheses ensure through
\fullref{lem:bias-variance}  that the importance sampling estimator in \cref{eqn:Gdef} is unbiased estimate of $Q^{b,\pi}(s_{i_1},a_{i_1}) -b(s_{i_1},a_{i_1})$ and bounded by $2G_{max}$ in absolute value with probability at least $1-n\delta'$ for all $n$ samples.
Combining this with \fullref{lem:StatisticalRates} we obtain that for a fixed target policy with probability at least $1-(n+1)\delta'$ we must have
\begin{align}
& \mathcal L(\widehat w,\pi) - \min_{\|w\|_2\leq W}\mathcal L(w,\rho,Q^{b,\pi} - b)\leq \frac{C_1 + C_2\ln\frac{1}{\delta'}}{n}, \quad \quad \quad
\end{align}
where
\begin{align}
C_1 & = \widetilde O\((W^2+G_{max}^2)d\) \\
C_2 & = \widetilde O\((W^2+G_{max}^2)\).
\end{align}
Finally, using \fullref{lem:Sigma2ERM} we conclude.
\end{proof}

\begin{lemma}[Union Bound]
\label[lemma]{lem:UnionBound}
Assume $0 \leq k - \underline k \leq \kappa$ where $\kappa$ is defined in \cref{eqn:kappa}.
Let $w^\star_k$ be as in \fullref{def:TransferError}, and let $\widehat w_k$ be the parameter computed during regression by  \cref{alg:solver} in line \ref{line:regression}. For some universal constant $c$ we have
\begin{align}
\Pro\(\|\widehat w_k -w^\star_k \|_{(\Sigma^{n})^{-1}} \leq c\sqrt{C_1+C_2\ln\frac{1}{\delta'} + \lambda W^2}\) \geq 1-\Big[\polyall\Big]^d \delta' - \frac{\delta}{4}\end{align}
\end{lemma}

\begin{proof}
Assume
\begin{align}
\label{eqn:kappa}
k - \underline k \leq \kappa \defeq \frac{(1-\gamma)\ln 2}{2\ln(8N^2K/\delta)\eta\( B + W \)}.
\end{align}
From we know that the all trajectories are bounded by $t_{max}$ with probability $>1-\delta/8$.
Define the unit ball  $\mathcal B \defeq \{ v \in \R^d \mid \| v \|_2 \leq 1 \}$.
As the target policy $\pi_k$ is a priori unknown we do a union bound over the possible vectors $v \defeq \sum_{i = \underline k}^{k-1} \widehat w_{i}$ which is a priori unknown and data-dependent. Since $v \in  (\kappa W)\mathcal B$,  consider the discretization $\mathcal D \subset (\kappa W)\mathcal B$ given in \fullref{lem:Discretization} with $\epsilon''$ to be determined in this proof.
The lemma ensures that for any $v \in (\kappa W)\mathcal B, \; \exists v' \in \mathcal D, \|v - v'\|_2 \leq \epsilon''$ and that $| \mathcal D| = (1+\frac{2\kappa W}{\epsilon''})^d$.

Now fix $v$ and let $\pi$ be the policy induced by $v$ and $\pi'$ the policy induced by $v'$.
Consider the empirical loss in \cref{eqn:Lhat} and let $\widehat w_v$ be the empirical minimizer corresponding to $\pi$ and $\widehat w_{v'}$ that corresponding to $\pi'$.
In addition let $w^{\star}_v$ be the minimizer corresponding
to $\pi$ and $w^{\star}_{v'}$ the minimizer corresponding to $\pi'$ of the true loss in \cref{eqn:Largmin}. We can write ($\Sigma$ is the expected covariance matrix which is fixed throughout the inner iterations)
\begin{align}
\label{eqn:Decomposition}
\| \widehat w_v - w^\star_v \|_{\Sigma} & \leq
\| \widehat w_v - \widehat w_{v'}\|_{\Sigma} + \| \widehat w_{v'} - w^\star_{v'}\|_{\Sigma} + \| w^\star_{v'} - w^\star_{v} \|_{\Sigma}.
\end{align}
We bound each term above.

In the event defined at the beginning of this proof that all trajectories are bounded in length by $t_{max}$ the importance sampling estimator in \cref{eqn:IS} is bounded in absolute value by $2G_{max} = \frac{2}{1-\gamma}(3B)$ for all policies $\pi$ satisfying \cref{eqn:policy-conditions} ($3B$ is the maximum absolute value of the reward including the bonus) and for all $n$ samples; In particular, the random timestep $t$ of any trajectory is bounded by $t_{max}$. Then \fullref{lem:StabilityEmpiricalQ} ensures that the importance sampling estimator for $\pi$ and $\pi'$ only differ by $8\epsilon''t_{max} G_{max}$. Plugging this into \fullref{lem:Stability} ensures
\begin{align}
\label{eqn:100}
\| \widehat w_{v'} - \widehat  w_{v} \|^2_{\Sigma} &  \leq   2 n(8\epsilon''t_{max} G_{max})  W + 2 \lambda W^2.
\end{align}
Likewise, \fullref{lem:StabilityQ} ensures that the true $Q$ values for $\pi$ and $\pi'$ on the optimistic MDP differ by at most $\frac{8\epsilon''}{1-\gamma}(2G_{max})$. Then \fullref{lem:Stability} ensures
\begin{align}
\label{eqn:101}
\| \widehat w^{\star}_{v'} - \widehat  w^{\star}_{v} \|^2_{\Sigma} &  \leq   2 n\( \frac{8\epsilon''}{1-\gamma}(2G_{max})\)  W + 2 \lambda W^2.
\end{align}
Setting $\frac{1}{\epsilon''} = \poly(N,K,\frac{1}{1-\gamma},G_{max},\ln(1/\delta))$ ensures that the rhs of \cref{eqn:100,eqn:101} is, say, $\leq 4\lambda W^2$ (we will have $\lambda > 1$ and $W > 1$) and also satisfies the requirement $4\epsilon''t_{max}\leq 1$ of  \fullref{lem:StabilityEmpiricalQ} ($t_{max}$ was defined at the beginning of the proof).

The $\epsilon''$ just computed determines the size of the discretization set $|\mathcal D|$ which is $\Big[\poly(n,\frac{1}{1-\gamma},G_{max})\Big]^d $.
A union bound over all $v' \in \mathcal D$ coupled with \fullref{lem:ISrates} ensures that with probability at least $$1-\Big[\polyall\Big]^d\frac{\delta'}{2}-\frac{\delta}{8}$$ for an appropriate universal constant $c$
\begin{align}
\forall v' \in \mathcal D: \quad \quad \quad \|\widehat w_{v'} - \widehat w_{v'}^\star \|^2_{\Sigma} \leq c \(C_1+C_2\ln\frac{1}{\delta'} + \lambda W^2 \).
\end{align}
Plugging back to \cref{eqn:Decomposition}
 concludes.
\end{proof}

\begin{lemma}[Validity of Confidence Intervals]
\label[lemma]{lem:ValidityOfConfidenceItervals}
With probability at least $1-\frac{\delta}{2}$ for all inner and outer iterations $n \in [N], k = 0,\dots,K-1$ of the algorithm it holds that
\begin{align}
\forall s\in \K^n, \forall a: \quad \quad \quad 	|Q^{n,\star}_k(s,a) - \Qhat^n_k(s,a) - b^n_\phi(s,a)| \leq b^n_\phi(s,a) \defeq \sqrt{\beta}\| \phi(s,a) \|_{(\widehat \Sigma^{n})^{-1}}
\end{align}
where $\beta$ is defined in \cref{eqn:beta}.
\end{lemma}
\begin{proof}
Define an appropriate
$\delta = \Big[\polyall\Big]^d \times \delta'$
and invoke \cref{lem:UnionBound}. A union bound over all inner and outer iterations ensures
\begin{align}
\Pro\(\forall n\in[N],\forall k=0,1,\dots,K-1: \|\widehat w^n_k -w^{n,\star}_k \|_{(\Sigma^{n})^{-1}} \leq \frac{1}{3}\sqrt{\beta}\) \geq 1-\frac{\delta}{4}
\end{align}
where
\begin{align}
\label{eqn:beta}
\beta & = \widetilde O \(C_1+C_2 d \times\ln[\polyall] + \lambda W^2 \) \\
& = \widetilde O(dW^2 + dG_{max}^2)
\end{align}
Combining the above result with \fullref{lem:Covariance} gives with probability $1-\delta/2$:
\begin{align}
|\phi(s,a)^\top (w_k^{n,\star} - \widehat w^n_k)| & \leq \| \phi(s,a) \|_{( \Sigma^{n})^{-1}}\| \widehat w^n_k - w^{n,\star}_k \|_{\Sigma^n} \\
& \leq 3\| \phi(s,a) \|_{(\widehat \Sigma^{n})^{-1}}\| \widehat w^n_k - w^{n,\star}_k \|_{\Sigma^n} \\
& = \sqrt{\beta}\| \phi(s,a) \|_{(\widehat \Sigma^{n})^{-1}}.
\end{align}
In other words, thanks to \fullref{lem:Covariance} we can use the empirical covariance in place of the full covariance. This implies that we can write the confidence intervals fully as a function of known quantities, in particular, using the empirical covariance matrix $\widehat \Sigma^n$ that \cref{alg:driver} maintains.

Using the definitions for the $Q$ values (still under the same event in known states):
\begin{align}
|\Qhat^n_k(s,a) + b^n_\phi(s,a) - Q^{n,\star}_k(s,a)| & = |\phi(s,a)^\top\widehat w^n_k + b^n_\phi(s,a) + b^n_\phi(s,a) - \phi(s,a)^\top w^{n,\star}_k - 2b^n_\phi(s,a)| \\
& \leq |\phi(s,a)^\top\widehat w^n_k - \phi(s,a)^\top w^{n,\star}_k| \\
& \leq \sqrt{\beta}\| \phi(s,a) \|_{(\widehat \Sigma^{n})^{-1}} \defeq b^n_\phi(s,a).
\end{align}
\end{proof}

\begin{lemma}[Trajectory Boundness]
\label[lemma]{lem:TrajectoryBoundness}
	Under the conditions on $\kappa$ in \fullref{lem:UnionBound}, all trajectories sampled by  \cref{alg:Sampler,alg:FeatureSampler} are bounded in length by $t_{max} = \frac{\ln(16N^2K/\delta)}{1-\gamma}$ with probability at least $1-\delta/8$.
\end{lemma}
\begin{proof}
	Then \fullref{lem:Update} ensures that the policies $\pi_k,\pi_{\underline k}$ take the form described in \cref{eqn:policy-conditions} with $\epsilon' = 1, \delta'=\delta/(8N^2K)$ and in particular, \fullref{lem:bias-variance} ensures that the trajectory lengths are all bounded by $t_{max} = \frac{\ln(16N^2K/\delta)}{1-\gamma}$ with probability at least $1-\delta/8$ after a union bound over $N$ trajectories collected possibly collected at each of the $K$ solver's iterations, times at most $N$ calls to the \textsc{Solver}, and a final union bound over the trajectories samples by \cref{alg:Sampler,alg:FeatureSampler}.
\end{proof}

\newpage
\section{Concentration of Bonuses}

The proof proceeds with the empirical covariance matrices since the determinant conditions is checked on the empirical matrices.

\textbf{Notation:} In this section for notational convenience the subscripts refer to the outer episode $n$; for example, we denote the covariance matrix with $\Sigma_n$ instead of $\Sigma^n$

\begin{lemma}[Bounding the Sum of Indicators]
\label[lemma]{lem:SumOfIndicators}
For any outer episode $n$ during the execution of the algorithm, let $\sigma(n)$ be the last episode smaller than $n$ where the bonus was updated. If $\pi^n$ takes an action where $b_\1$ is nonzero in a state  $s\notin \K^n$ and $\lambda \geq 1$ then under the event of \fullref{lem:SumOfBonuses} we have
\begin{align}
\sum_{n=1}^N\E_{(s,a) \sim \pi^n \mid s_0}b^{\sigma(n)}_\1(s,a) \defeq \frac{3}{1-\gamma} \sum_{n=1}^N\E_{s \sim \pi^n \mid s_0}\unkn &  \leq \widetilde O\(\frac{\sqrt{\beta Nd}}{1-\gamma}\).
\end{align}
\end{lemma}
\begin{proof}
\begin{align}
\sum_{n=1}^N\E_{s\sim \pi^n \mid s_0}\unkn & = \sum_{n=1}^N\E_{ (s,a)\sim \pi^n \mid s_0}\1\{ \sqrt{\beta}\|\phi(s,a)\|_{\widehat \Sigma^{-1}_{\sigma(n)}} \geq 1 \} \\
& \leq  \sum_{n=1}^N\E_{(s,a)\sim \pi^n \mid s_0}\sqrt{\beta}\|\phi(s,a)\|_{\widehat \Sigma^{-1}_{\sigma(n)}}.
\end{align}
Finally \fullref{lem:SumOfBonuses} concludes.
\end{proof}

\begin{lemma}[Bounding the Sum of Bonuses]
\label[lemma]{lem:SumOfBonuses}
For any outer episode $n$ during the execution of the algorithm, let $\sigma(n)$ be the last episode smaller than $n$ where the bonus was updated.
If $\lambda \geq 1$ then with probability at least $1-\delta'$
\begin{align}
\sum_{n=1}^N\E_{(s,a)\sim \pi^n \mid s_0}b_{\phi}^{\sigma(n)}(s,a) & = \sqrt{\beta}\sum_{n=1}^{N} \E_{(s,a)\sim \pi^n \mid s_0}\| \phi(s,a) \|_{\widehat \Sigma_{\sigma(n)}^{-1}} \\
 & \leq \sqrt{\beta} O\( \sqrt{ND} + \ln(1/\delta') \) = \widetilde O(\sqrt{\beta Nd}),
\end{align}
where $D$ is defined in \cref{eqn:Ddef}. \end{lemma}
\begin{proof}

We can write
\begin{align}
\sum_{n=1}^{N}\E_{(s,a)\sim \pi^n \mid s_0} b_{\phi}^{\sigma(n)}(s,a) & = \sqrt{\beta}\sum_{n=1}^{N} \E_{(s,a)\sim \pi^n \mid s_0}\| \phi(s,a) \|_{\widehat \Sigma^{-1}_{\sigma(n)}}.
\end{align}

We need to bound the summation
for any realization of the sequence of $(\pi^n,\widehat \Sigma_{\sigma(n)})$.
Define the random dataset $\mathcal D_{1:n}$ containing all the information (i.e., the realization of the random variables) at the \emph{beginning} of iteration $n$ of the algorithm. Conditioning on $\mathcal D_{1:n}$ fixes the policy $\pi^n$ and the covariance $\widehat \Sigma_{\sigma(n)}$ and the distribution over $\phi$.

Notice that in each episode $n$ the random variable $\phi$ and the collected feature $\phi_n$ are identically distributed when conditioned on $\mathcal D_{1:n}$ (since their distribution is uniquely determined by the policy $\pi^n$ in that episode, which is fixed under the conditioning on $\mathcal D_{1:n}$). Therefore we can define the `noise' in the sampled feature
\begin{align}
\xi_n = \E_{(s,a)\sim \pi^n \mid s_0} \Big[\|\phi(s,a)\|_{\widehat  \Sigma^{-1}_{\sigma(n)}} \mid \mathcal D_{1:n}\Big] -  \|\phi_{n}\|_{\widehat \Sigma^{-1}_{\sigma(n)}}
\end{align}
and write
\begin{align}
\label{eqn:Adef}
A \defeq \sum_{n=1}^N \E_{(s,a)\sim \pi^n \mid s_0}\Big[\|\phi(s,a)\|_{\widehat  \Sigma^{-1}_{\sigma(n)}} \mid \mathcal D_{1:n}\Big] & = \sum_{n=1}^N  \|\phi_{n}\|_{\widehat  \Sigma^{-1}_{\sigma(n)}} + \sum_{n=1}^N  \xi_n \\
& \leq \sqrt{N\sum_{n=1}^N  \|\phi_{n}\|^2_{\widehat  \Sigma^{-1}_{\sigma(n)}}} + \sum_{n=1}^N  \xi_n
\end{align}
The first summation on the rhs is bounded by \fullref{lem:PotentialArgument} by $D$; it remains to bound the sum of the noise terms.
Conditioned on $\mathcal D_{1:n}$, the noise  $\xi_n$ is mean-zero. Summing over the conditional second moments gives:
\begin{align}
\sum_{n=1}^N\E_{(s,a) \sim \pi^n \mid s_0} \Big[\xi^2_n \mid \mathcal D_{1:n}\Big] & =  \sum_{n=1}^N \E_{(s,a) \sim \pi^n \mid s_0} \Big[\|\phi(s,a) \|^2_{\widehat \Sigma^{-1}_{\sigma(n)}} \mid \mathcal D_{1:n}\Big] \\
& \leq \sum_{n=1}^N \E_{(s,a) \sim \pi^n \mid s_0} \Big[\|\phi(s,a) \|_{\widehat  \Sigma^{-1}_{\sigma(n)}} \mid \mathcal D_{1:n}\Big] = A.
\end{align}
The last step follows because if $\lambda \geq 1$ and $\| \phi(\cdot,\cdot) \|_2 \leq 1$ we have $\|\phi(\cdot,\cdot) \|_{\widehat  \Sigma^{-1}_{\sigma(n)}} \leq \|\phi(\cdot,\cdot) \|_{2} \leq 1$ giving $\|\phi(\cdot,\cdot) \|_{\widehat  \Sigma^{-1}_{\sigma(n)}}^2 \leq \|\phi(\cdot,\cdot) \|_{\widehat  \Sigma^{-1}_{\sigma(n)}} $.
Now, \fullref{lem:Bernstein}  gives with probability at least $1-\delta'$ for some constant $c$
\begin{align}
\sum_{n=1}^N \xi_n  & \leq c \times \( \sqrt{2 \sum_{n=1}^N \E_{(s,a) \sim \pi^n \mid s_0} [\xi^2_n \mid \mathcal D_{1:n}] \ln(1/\delta')} + \frac{\ln(1/\delta')}{3}\) \\
& = c \times \( \sqrt{2 A\ln(1/\delta')} + \frac{\ln(1/\delta')}{3} \).
\end{align}
Combining with \cref{eqn:Adef} we have shown that with probability at least $1-\delta'$ we must have the following relation
\begin{align}
A \leq \sqrt{ND} + c \times \( \sqrt{2A\ln(1/\delta')} + \frac{\ln(1/\delta')}{3} \).
\end{align}
Solving for $A$ finally gives with high probability
\begin{align}
A   = O \( \sqrt{ND} +\ln(1/\delta')\).
\end{align}
\end{proof}

The following lemma is used to claim that whenever the determinant condition is violated (triggering a new call to the \textsc{Solver}) then the condition is not violated by much.
\begin{lemma}[Maximum Determinant Ratio]
If $\lambda \geq 1$ and $\|\phi(\cdot,\cdot)\|_2 \leq 1$ then $\det\(\widehat \Sigma_{n}\) \leq 4\det\(\widehat \Sigma_{\underline n}\)$.
\end{lemma}
\begin{proof}
If $\det\(\widehat \Sigma_{n}\) \leq 2\det\(\widehat \Sigma_{\underline n}\)$ the statement holds; if $\det\(\widehat \Sigma_{n}\) > 2\det\(\widehat \Sigma_{\underline n}\)$ then by construction we must have $\det\(\widehat \Sigma_{n-1}\) \leq 2\det\(\widehat \Sigma_{\underline n}\)$ (as the algorithm switches to a new policy once such condition is violated).
Use \fullref{lem:DeterminantRatio} and recall $ \| \phi_{n-1}\|^2_{\widehat \Sigma_{n-1}^{-1}} \leq \| \phi_{n-1} \|^2_{2} \leq 1$ for $\lambda \geq 1$ to write
\begin{align}
\det\(\widehat \Sigma_{n}\) = \det\(\widehat \Sigma_{n-1}\) \( 1 + \| \phi_{n-1}\|^2_{\widehat\Sigma_{n-1}^{-1}}\) \leq 2\det\(\widehat \Sigma_{n-1}\) \leq 4\det\(\widehat \Sigma_{\underline n}\).
\end{align}
\end{proof}

The following lemma is key. It implicitly quantifies the loss due to the delayed update of the covariance matrix; the effect of such delay are rather mild, as they only affect a numerical constant.

\begin{lemma}[Trace to LogDeterminant]
\label[lemma]{lem:TrLogDet}
Let $\Sigma$ be a positive define matrix and let $M$ be a symmetric positive semidefinite matrix. Let $\Sigma' = \Sigma + M$. Then if $\det(\Sigma') \leq 4\det(\Sigma)$ we have
\begin{align}
\ln\det(\Sigma') \geq \ln\det(\Sigma) + \frac{1}{3} \Tr(M\Sigma^{-1}).
\end{align}
\end{lemma}
\begin{proof}
We have
\begin{align}
\det(\Sigma') & = \det(\Sigma + M) \\
& = \det(\Sigma^{\frac{1}{2}}) \det(I + \Sigma^{-\frac{1}{2}}M \Sigma^{-\frac{1}{2}})  \det(\Sigma^{\frac{1}{2}})\\
& = \det(\Sigma) \det(I + \Sigma^{-\frac{1}{2}}M \Sigma^{-\frac{1}{2}}).\end{align}
Denote with $\lambda_1,\dots,\lambda_d$ the eigenvalues of $\Sigma^{-\frac{1}{2}}M \Sigma^{-\frac{1}{2}}$.
We must have by hypothesis
\begin{align}
\label{eqn:DetRatio}
4 \geq \frac{\det(\Sigma') }{\det(\Sigma) } = \det(I + \Sigma^{-\frac{1}{2}}M \Sigma^{-\frac{1}{2}}) = \prod_{j=1}^d (1+\lambda_j).
\end{align}
Taking $\ln$ gives:
\begin{align}
\ln 4 \geq  \sum_{j=1}^d \ln(1+\lambda_j).
\end{align}
Since all $\lambda_j$'s must be positive, the $\ln$ terms in the rhs above are positive, and each must satisfy
\begin{align}
\ln 4 \geq \ln(1+\lambda_j), \quad \quad \quad \forall j \in [d]
\end{align}
and so in particular (by exponentiating the above display)
\begin{align}
3 \geq \lambda_j, \quad \quad \quad \forall j \in [d]
\end{align}
which allows us to use the following inequality
\begin{align}
\ln(1+\lambda_j) \geq \frac{1}{3}\lambda_j.
\end{align}
Going back to \cref{eqn:DetRatio} (and again taking $\ln$) gives
\begin{align}
\ln\det(\Sigma') & = \ln\det(\Sigma) + \sum_{j=1}^d \ln(1+\lambda_j) \\
& \geq \ln\det(\Sigma) + \frac{1}{3} \sum_{j=1}^d \lambda_j \\
& \geq \ln\det(\Sigma) + \frac{1}{3} \Tr(\Sigma^{-\frac{1}{2}}M \Sigma^{-\frac{1}{2}}) \\
& = \ln\det(\Sigma) + \frac{1}{3} \Tr(M\Sigma^{-1}).
\end{align}
\end{proof}

\begin{lemma}[Potential Argument]
\label[lemma]{lem:PotentialArgument}
Let  ${\underline n}_1,{\underline n}_2,\dots,{\underline n}_{last}$ be the indexes in the sequence $n=1,\dots,N$ where the bonus gets updated and let $\sigma(n)$ be the last episode smaller than $n$ where the bonus was updated, i.e.,
\begin{align}
\widehat\Sigma_{{\underline n}_{i+1}} & = \widehat\Sigma_{{\underline n}_{i}} + \sum_{n={\underline n}_i}^{{\underline n}_{i+1}-1} \phi_n \phi_n^\top \\
\widehat\Sigma_{n} & = \widehat\Sigma_{\sigma(n)}, \quad {\underline n}_{i} \leq n < {\underline n}_{i+1} \\
\det(\widehat\Sigma_{{\underline n}_{i+1}}) & \leq 4 \det(\widehat\Sigma_{{\underline n}_i}).
\end{align}
We have
\begin{align}
\label{eqn:Ddef}
\sum_{n=1}^N \| \phi_n \|^2_{\widehat\Sigma^{-1}_{\sigma(n)}}  \leq 3 \( \ln\det(\widehat\Sigma_{N+1}) - \ln\det(\widehat\Sigma_{1}) \) \defeq D = \widetilde O(d).
\end{align}
\end{lemma}
\begin{proof}
Let ${\underline n}_{last}$ be the index of the last switch.
Use \fullref{lem:TrLogDet} twice with the following inputs (notice that the determinant ratio condition is satisfied in both cases)
\begin{align}
\Sigma' = \widehat\Sigma_{{\underline n}_{i+1}}, \quad \Sigma = \widehat\Sigma_{{\underline n}_{i}}, \quad M = \sum_{n={\underline n}_i}^{{\underline n}_{i+1}-1} \phi_n \phi_n^\top \\
\Sigma' = \widehat \Sigma_{N+1}, \quad \Sigma = \widehat\Sigma_{{\underline n}_{last}}, \quad M = \sum_{n={\underline n}_{last}}^{N} \phi_n \phi_n^\top
\end{align}
to obtain
\begin{align}
\ln\det(\widehat\Sigma_{{\underline n}_{i+1}}) - \ln\det(\widehat\Sigma_{{\underline n}_{i}}) & \geq   \frac{1}{3} \Tr(\sum_{n={\underline n}_i}^{{\underline n}_{i+1}-1} \phi_n \phi_n^\top \widehat\Sigma_{{\underline n}_i}^{-1}) \\
& = \frac{1}{3} \sum_{n={\underline n}_i}^{{\underline n}_{i+1}-1}\Tr( \phi_n \phi_n^\top \widehat\Sigma_{{\sigma(n)}}^{-1})
\end{align}
and likewise
\begin{align}
\ln\det(\widehat\Sigma_{N+1}) - \ln\det(\widehat\Sigma_{{\underline n}_{last}})  & \geq \frac{1}{3} \sum_{n=n_{last}}^{N} \Tr( \phi_n \phi_n^\top \widehat\Sigma_{\sigma(n)}^{-1})
\end{align}
Summing over the switches, recalling ${\underline n}_1 = 1$ and adding the above display gives (after cancelling the terms in the telescoping sum)
\begin{align}
\ln\det(\widehat\Sigma_{N+1}) -  \ln\det(\widehat\Sigma_{{\underline n}_{1}})
& \geq \frac{1}{3} \sum_{n=1}^{N}\Tr( \phi_n \phi_n^\top \widehat\Sigma_{\sigma(n)}^{-1}).
\end{align}
Finally, consider
\begin{align}
\| \phi_n \|^2_{\Sigma^{-1}_{\sigma(n)}} & = \phi_n^\top\widehat\Sigma^{-1}_{\sigma(n)} \phi_n \\
& =\Tr \(\phi_n^\top\widehat\Sigma^{-1}_{\sigma(n)} \phi_n \) \\
& =\Tr \(\phi_n \phi_n^\top\widehat\Sigma^{-1}_{\sigma(n)}  \).
\end{align}
Combining with the prior display concludes.
\end{proof}

\begin{lemma}[Determinant Ratio]
\label[lemma]{lem:DeterminantRatio}
If $\Sigma^+ = \Sigma + \phi\phi^\top $ and $\Sigma$ is strictly symmetric positive definite then $\det(\Sigma^+) = \det(\Sigma)\( 1 + \| \phi\|^2_{\Sigma^{-1}}\)$.
\end{lemma}
\begin{proof}
The inverse of $\Sigma$ exists because $\Sigma$ is strictly positive definite. We can write
\begin{align}
\det(\Sigma^+) & = \det(\Sigma + \phi\phi^\top) \\
& = \det\Bigg[ \Sigma^{\frac{1}{2}} \(I +  \Sigma^{-\frac{1}{2}} \phi\phi^\top \Sigma^{-\frac{1}{2}} \)\Sigma^{\frac{1}{2}} \Bigg] \\
& = \det(\Sigma)\det \(I +  \Sigma^{-\frac{1}{2}} \phi\phi^\top \Sigma^{-\frac{1}{2}} \).
\end{align}
We use the matrix determinant lemma to continue and write
\begin{align}
& = \det(\Sigma) \( 1 +  \phi^\top \Sigma^{-\frac{1}{2}}\Sigma^{-\frac{1}{2}} \phi  \) \\
& = \det(\Sigma) \( 1 +  \|\phi\|^2_{\Sigma^{-1}} \).
\end{align}
\end{proof}

\begin{lemma}[Number of Switches]
\label[lemma]{lem:Switches}
Using the same notation as \cref{lem:PotentialArgument} we have that the number of times the bonus is updated is $\widetilde O(d)$.
\end{lemma}
\begin{proof}
Notice that $\det(\widehat\Sigma_1) \geq \det(\widehat\Sigma_0) = \lambda_{min}^d$ (for the definition of $\lambda_{min}$, please see \cref{eqn:lambdamin}) and $\det(\widehat\Sigma_N) \leq (\lambda_{min} + \frac{(N+1)}{d})^d$ (see proof of lemma 11 in \citep{Abbasi11}). Let  ${\underline n}_1,{\underline n}_2,\dots,{\underline n}_{last}$ be the indexes in the sequence $n=1,\dots,N$ where the bonus $b^n$ gets updated. Every time the bonus is updated  we have
\begin{align}
\det(\widehat\Sigma_{\underline n_{i+1}}) \geq 2\det(\widehat\Sigma_{\underline n_{i}}).
\end{align}
Let $S$ denote the number of times the bonus is updated.
By induction,
\begin{align}
\(\lambda_{min} + \frac{(N+1)}{d}\)^d \geq \det(\Sigma_{\underline n_{last}}) \geq 2^S \det(\Sigma_1) \geq 2^S \lambda^d_{min}
\end{align}
It follows that
\begin{align}
S \leq d\ln_2 \(1 + \frac{(N+1)}{d\lambda_{min}}\) = \widetilde O(d)
\end{align}
\end{proof}

\newpage
\section{Inverse Covariance Matrix Estimation}

\begin{lemma}[Concentration of Inverse Covariances]
\label[lemma]{lem:Covariance}
Let $\mu_i$ be the conditional distribution of $\phi$ given the sampled $\phi_1,\dots,\phi_{i-1}$.  Assume $\| \phi \|_2 \leq 1$ for any realization of the vector.
Define $ \Sigma = \frac{1}{n}\sum_{i=1}^n \E_{\phi \sim \mu_i} \phi\phi^\top$. If
\begin{align}
\lambda \geq \lambda_{min} =  \Omega (d\ln(n/\delta'')).
\end{align}
where $\lambda_{min}$ is defined in \cref{eqn:lambdamin} then we have
\begin{align}
\Pro\(\forall n \geq 1,  \quad
\frac{3}{1}\(n\Sigma + \lambda I\)^{-1} \succeq   \( \sum_{i=1}^n \phi_i\phi_i^\top + \lambda I\)^{-1}  \succeq  \frac{3}{5}\(n\Sigma + \lambda I\)^{-1} \) \geq 1-\delta''.
\end{align}
In the same event as above the following event must hold as well
\begin{align}
\forall n \geq 1,  \quad
\frac{1}{3}\(n\Sigma + \lambda I\) \preceq   \( \sum_{i=1}^n \phi_i\phi_i^\top + \lambda I\)  \preceq  \frac{5}{3}\(n\Sigma + \lambda I\).
\end{align}
\end{lemma}
\begin{proof}
Consider any $x$ such that $\| x\|_2 = 1$. Let $\Sigma_i = \E_{\phi \sim \mu_i} \phi\phi^\top$ and $\Sigma \defeq \frac{1}{n}\sum_{i=1}^n  \Sigma_i$. We have
\begin{align}
\E_{\phi \sim \mu_i} x^\top  \phi\phi^\top x = \E_{\phi \sim \mu_i}  \( x^\top \phi \)^2 = x^\top\Sigma_i x.
\end{align}
The random variable $\( x^\top \phi \)^2$, $\phi \sim \mu_i$ is positive with maximum value $\( x^\top \phi \)^2 \leq \|x \|^2_2\|\phi\|^2_2 \leq 1$ and mean $x^\top\Sigma_i x$; therefore the conditional variance is at most $x^\top\Sigma_i x$ as well, as we show below
\begin{align}
\Var_{\phi \sim \mu_i} (\phi^\top x)^2 \leq \E_{\phi \sim \mu_i} (\phi^\top x)^2 = x^\top\Sigma_i x.
\end{align}
Now \fullref{lem:Bernstein} gives with probability at least $1-\delta'$ for some constant $c$
\begin{align}
\Big|\frac{1}{n}\sum_{i=1}^n \Big[ \( x^\top \phi_i \)^2 - x^\top \Sigma_i x \Big] \Big| = \Big|\frac{1}{n}\sum_{i=1}^n \( x^\top \phi_i \)^2 - x^\top \Sigma x \Big| \leq  c\( \sqrt{2\frac{x^\top\Sigma x}{n} \ln(2/\delta')} + \frac{\ln(2/\delta')}{3n} \).
\end{align}

We require
\begin{align}
\label{eqn:General}
c\( \sqrt{2\frac{x^\top\Sigma x}{n} \ln(2/\delta')} + \frac{\ln(2/\delta')}{3n}\) \leq 	\frac{1}{2}\( x^\top\Sigma x + \frac{\lambda}{n} \).
\end{align}

We will show that if
\begin{align}
\lambda \geq \Omega (\ln(1/\delta'))
\end{align}
then \cref{eqn:General} holds for any fixed value of $n$.

\textbf{Case $x^\top\Sigma x \leq \frac{\lambda}{n}$}.
In this case it is sufficient to satisfy for some constants $c',c''$
\begin{align*}
\frac{\ln(2/\delta')}{3n} & \leq c'\( \frac{\lambda}{n} \) \quad  \longleftrightarrow \quad \Omega(\ln(1/\delta'))\leq \lambda  \\
\sqrt{2\frac{\lambda}{n^2} \ln(2/\delta')} & \leq c''\(\frac{\lambda}{n} \)\quad  \longleftrightarrow \quad \Omega(\ln(1/\delta'))\leq \lambda.
\end{align*}
\textbf{Case $x^\top\Sigma x > \frac{\lambda}{n}$}.
In this case to satisfy proceed as in the above display (first equation). For the second equation it is sufficient to satisfy for some constant $c'''$
\begin{align*}
\sqrt{2\frac{x^\top\Sigma x}{n} \ln(2/\delta')} & \leq c'''\( x^\top\Sigma x  \)\quad  \longleftrightarrow \quad \frac{\ln(2/\delta')}{(x^\top\Sigma x)n}  \leq O(1).
\end{align*}
Using the condition $x^\top\Sigma x > \frac{\lambda}{n}$ we can conclude that satisfying $ \Omega( \ln(1/\delta'))\leq \lambda$ suffices.

This ultimately implies that for any fixed $x$ such that $\|x\|_2 = 1$ we have
\begin{align}
\label{eqn:Result}
\Big|x^\top \( \frac{1}{n} \sum_{i=1}^n \phi_i\phi_i^\top - \Sigma\) x \Big| \leq	\frac{1}{2}  \(x^\top\Sigma x + \frac{\lambda}{n} \) = \frac{1}{2} x^\top \(\Sigma + \frac{\lambda}{n}I \) x.
\end{align}
with probability at least $1-\delta'$. Define $ \partial \mathcal B = \{\| x \|=1\}$; using a standard discretization argument, (e.g., lemma 5.2 in \citep{vershynin2010introduction}) we have that
\begin{align}
\forall \epsilon > 0, \;  \exists \mathcal B_{\epsilon} \subseteq \mathcal B \quad \text{such that} \quad \forall x \in \mathcal B, \exists x' \in \mathcal B_\epsilon \subseteq \mathcal B \quad \text{such that} \quad  \| x - x' \|_2 \leq \epsilon
\end{align}
and
\begin{align}
|\mathcal B_{\epsilon'}| \leq \(\frac{3}{\epsilon'}\)^d \defeq \mathcal N.
\end{align}
Therefore, applying the result of \cref{eqn:Result} to any such $x' \in \mathcal B_\epsilon$ gives after a union bound over the $x'$ and the number of samples $n$ that
with probability at least $1-n\mathcal N \delta' \defeq 1 -\delta''$ we have that
\begin{align}
\forall n, \; \forall x'\in \mathcal B_{\epsilon'},  \quad \quad \quad \Big|(x')^\top \(\frac{1}{n}\sum_{i=1}^n \phi_i\phi_i^\top\) x' - (x')^\top\Sigma x' \Big|\leq  \frac{1}{2} (x')^\top \(\Sigma  + \frac{\lambda}{n}I\)x'.
\end{align}
Now for any $x \in \partial \mathcal B$ consider the closest $x'\in \mathcal B_{\epsilon'}$. We have that for an spd matrix $A$ such that $\| A \|_2 \leq 1$
\begin{align}
x^\top A x - (x')^\top A x' & = x^\top A x - (x')^\top A x + (x')^\top A x - (x')^\top A x' \\
& = (x- x')^\top A x + (x')^\top A (x- x') \\
& \leq 2\epsilon' \|A\|_2\max\{\|x\|_2,\|x'\|_2\} \\
& \leq  2\epsilon'.
\end{align}
Apply this to the case $A = \Sigma$ and $A = \frac{1}{n}\sum_{i=1}^n\phi_i\phi_i^\top$ (notice that $\| \Sigma\|_2\leq 1$ and $\|  \frac{1}{n}\sum_{i=1}^n\phi_i\phi_i^\top \|_2\leq 1$ follow from hypothesis)  to obtain
\begin{align}
\Bigg| x^\top \( \frac{1}{n} \sum_{i=1}^n \phi_i\phi_i^\top \) x - (x')^\top \( \frac{1}{n} \sum_{i=1}^n \phi_i\phi_i^\top\) x' \Bigg| \leq 2  \epsilon' \\
\Bigg| x^\top \Sigma x - (x')^\top \Sigma x' \Bigg| \leq 2  \epsilon'.
\end{align}

This implies that, if
\begin{align}
\label{eqn:lambdamin}
\lambda \geq \Omega\( \ln\(\frac{2n\mathcal N}{\delta''}\)\)   \defeq \lambda_{min}
\end{align}
then
\begin{align}
\forall n, \; \forall x\in \mathcal B, \quad \quad \quad \Bigg|x^\top \Bigg[\(\frac{1}{n}\sum_{i=1}^n \phi_i\phi_i^\top + \frac{\lambda}{n}I\) - \(\Sigma + \frac{\lambda}{n}I \) \Bigg] x \Bigg|& \leq  \frac{1}{2} x^\top \(\Sigma + \frac{\lambda}{n}I \)x + 4\epsilon' \\
& \leq \frac{2}{3}x^\top\(\Sigma + \frac{\lambda}{n} I \)x
\end{align}
by setting $\epsilon' = \mathcal O\(\frac{1}{n}\)$ (as we set $\lambda > 1$ and in addition $x \in \partial \mathcal B$).
This implies
\begin{align}
\frac{1}{3}\(\Sigma + \frac{\lambda}{n}I\) \preceq   \frac{1}{n}\sum_{i=1}^n \phi_i\phi_i^\top + \frac{\lambda}{n}& \preceq  \frac{5}{3}\(\Sigma  + \frac{\lambda}{n}I\)
\end{align}
and finally the thesis.
\end{proof}

\newpage
\section{Technical Results}
\begin{lemma}[Policy Form on Known Set]
\label[lemma]{lem:Update}
Fix $n$ (the outer iteration index) and the bonus $b$ in that outer iteration. Let $k$ be an inner episode of the algorithm and let $\underline k$ be the last time data were collected. The policy $\pi_k$ computed by the algorithm in the inner episode $k$ reads in any known state $s \in \K^n$ for some $w_{\underline k},\dots,w_{k-1}$:
\begin{align}
\label{eqn:Update}
\pi_{k}(a\mid s) & = \pi_{\underline k}(a \mid s) \times \frac{e^{c(s,a)}}{\sum_{a'}\pi_{\underline k}( a' \mid s) e^{c(s,a')}}, \quad \quad \text{where} \quad  c(s,a) = \eta\sum_{i = \underline k}^{k-1} \big[b(s,a) + \phi(s,a)^\top  \widehat w_{i}\big].
\end{align}
\end{lemma}

\begin{proof}

Assume $k > \underline k$, otherwise the the statement is trivially true.
The update rule reads
\begin{align}
\pi_{k}( \cdot \mid s) & \propto \pi_{k-1}( \cdot \mid s)e^{\eta \Qhat_{k-1}( \cdot \mid s)} \\
& = \pi_{k-1}( \cdot \mid s)e^{\eta \big[ \phi(s, \cdot)^\top \widehat w_{k-1} + b(s, \cdot) \big]}
\end{align}
Using induction gives
\begin{align}
\pi_{k}( \cdot\mid s) & \propto \pi_{\underline k}( \cdot \mid s)  \prod_{I = \underline k}^{k-1} e^{\eta \big[ \phi(s,a)^\top \widehat w_{i} + b(s, \cdot) \big]} \\
& \propto \pi_{\underline k}( \cdot \mid s) \times e^{ c(s,\cdot)}.
\end{align}
Normalization concludes.
\end{proof}

\begin{lemma}[$\Sigma$-norm to Excess Risk] Fix $ \lambda > 0$.
Define
\label[lemma]{lem:Sigma2ERM}
\begin{align}
\L(w) & = \frac{1}{2}\E_{(x,y)} \( x^\top w - y\)^2 \\
w^\star & \in \argmin_{\| w \|_2 \leq W} \L(w).
\end{align}
Then for any scalar $M > 0$
\begin{align}
\| w - w^\star\|^2_{(M\E_{(x,y)} x x^\top +  \lambda I)}\leq  2M\( \L(w) - \L(w^\star)\) +   \lambda \| w - w^\star \|_2^2.
\end{align}
\end{lemma}
\begin{proof}
We write $\E$ in place of $\E_{(x,y)}$ for short.
The optimality condition reads (for any feasible $w$)
\begin{align}
\E\( x^\top w^\star - y\)x^\top (w - w^\star) \geq  0.
\end{align}
Therefore
\begin{align}
2\Big[\L(w) - \L(w^\star) \Big] & = \E \( x^\top w - y\)^2 - \E \( x^\top w^\star - y\)^2 \\
& = \E \Big[ \( x^\top w - y\) - \( x^\top w^\star - y\) \Big]\Big[  x^\top w - y +  x^\top w^\star - y \Big] \\
& = \E \Big[ x^\top\( w - w^\star \)  \Big]\Big[  x^\top w - y +  x^\top w^\star - y \Big] \\
& = \E \Big[ x^\top\( w - w^\star \)  \Big]\Big[  x^\top (w - w^\star) +  x^\top w^\star - y +  x^\top w^\star - y \Big] \\
& = \( w - w^\star \)^\top \E \( x x^\top\)\( w - w^\star \)   + 2\E \Big[ x^\top\( w - w^\star \)  \Big]\Big[ x^\top w^\star - y  \Big] \\
& \geq \( w - w^\star \)^\top \E \( x x^\top\)\( w - w^\star \) \\
& = \( w - w^\star \)^\top \Big[ \E  x x^\top + \frac{ \lambda}{M} I \Big]\( w - w^\star \) - \frac{ \lambda}{M} \| w - w^\star \|_2^2.
\end{align}

The inequality follows from the optimality conditions in the prior display.
\end{proof}

\begin{lemma}[Stability of the Loss Minimizer]
\label[lemma]{lem:Stability}
Let
\begin{align}
\L(w) & = \frac{1}{2}\E_{(x,y)} \( x^\top w - y\)^2, \quad
w_\star \in \argmin_{\| w \|_2 \leq W} \L(w) \\
\L'(w) & = \frac{1}{2}\E_{(x,y)} \( x^\top w - y - f(y)\)^2, \quad
w'_\star \in \argmin_{\| w \|_2 \leq W} \L'(w).
\end{align}
If $ \lambda > 0$ and the perturbation $|f(y)| \leq \epsilon_f$ for every $y$ and $\Sigma = \E_{(x,y)} xx^\top$ and $M > 0$ then
\begin{align}
\| w'_\star - w_\star \|^2_{M\Sigma +  \lambda I} &  \leq   2 M\epsilon_f  W + 2 \lambda W^2.
\end{align}
\end{lemma}

\begin{proof}

Define $y' = y + f(y)$ for short. We write $\E$ instead of $\E_{(x,y)}$ for brevity.
The optimality conditions at the minimizers $w_\star$ and $w_\star'$ for the real losses $ \L(w)$ and $ \L'(w)$ read
\begin{align}
\frac{\partial}{\partial w} \L(w_\star)(w - w_\star) & =\E (x^\top w_\star - y)x^\top (w - w_\star) \geq 0 \\
\frac{\partial}{\partial w} \L'(w'_\star)(w - w_\star') & = \E (x^\top w_\star' - y')x^\top (w - w_\star') \geq 0.
\end{align}
Take the first condition and evaluate it at $w = w_\star'$ to write
\begin{align}
0 & \leq \E (x^\top w_\star - y)x^\top (w_\star' - w_\star)  \\
& = \E \(x^\top w_\star' + x^\top (w_\star - w_\star') - y' + (y' - y) \) x^\top(w_\star' - w_\star) \\
& = \E \(x^\top w_\star' - y' \)x^\top (w_\star' - w_\star) \\
& + \E \( x^\top (w_\star - w_\star') + (y' - y) \) x^\top (w_\star' - w_\star)
\end{align}
The first term in the above rhs must be negative due to the second optimality condition (for $w = w_\star$) in the previous display; therefore, at the very least the second term in the rhs above must be positive
\begin{align}
0 & \leq \E \( x^\top (w_\star - w_\star') + (y' - y) \) x^\top (w_\star' - w_\star) \\
& = \E (w_\star' - w_\star)^\top xx^\top (w_\star - w_\star') + \E (y' - y)  x^\top (w_\star' - w_\star)
\end{align}
and next
\begin{align}
M(w_\star' - w_\star)^\top \E xx^\top (w'_\star - w_\star)  & \leq  M\E (y' - y)  x^\top (w_\star' - w_\star)
\end{align}
and finally
\begin{align}
(w_\star' - w_\star)^\top \(M\E xx^\top  + \lambda I\)(w'_\star - w_\star)  & \leq  M\E (y' - y)  x^\top (w_\star' - w_\star) + \lambda \|w_\star - w_\star' \|_2^2 \\
& \leq  2 M\epsilon_f  W + 2 \lambda W^2.
\end{align}
\end{proof}

\begin{lemma}[Statistical Rates for Linear Regression; Theorem 1 in \citep{mehta2017fast}]
\label[lemma]{lem:StatisticalRates}
With $z = (\phi,y)$ let $l_w(z) \mapsto \frac{1}{2}(\phi^\top w -y)^2$. Assume $Z \sim P$ and $\|\phi\|_2\leq 1,\| w \|_2\leq W,|y|\leq y_{max}$. Let $\widehat w$ be the empirical risk minimizer with $n$ i.i.d. samples from $P$ and let $w^\star = \argmin_{\| w \|_2\leq W}\E_{Z\sim P}l_{w}(Z)$.
With probability at least $1-\delta'$ we have
\begin{align}
\E_{Z\sim P}[l_{\widehat w}(Z) - l_{w^\star}(Z)] \leq \frac{1}{n}\Big[ 32(W+y_{max})^2\times \Big[d \ln (16(W+y_{max})(2W)n) + \ln\frac{1}{\delta'}\Big]+1 \Big]
\end{align}
\end{lemma}
\begin{proof}
The maximum value the loss can take is $L^2_{max} = (W + y_{max})^2$. The statement then follows as an application of Theorem 1 in \citep{mehta2017fast} to linear regression, which is $1/(4L_{max}^2)$-exp-concave (end of section $3$ in \citep{mehta2017fast}.
\end{proof}

\begin{lemma}[Discretization of Euclidean Ball]
\label[lemma]{lem:Discretization}
The Euclidean sphere $R\mathcal B = \{x \mid \| x \|_2 \leq R, x \in \R^d \}$ with radius $R$ equipped with the Euclidean metric admits a discretization for every $\epsilon>0$: $\mathcal D_{\epsilon} = \{ y_1,\dots,y_{\mathcal N_\epsilon} \} \subseteq R\mathcal B$ with
\begin{align}
\mathcal N_\epsilon \leq \(1+\frac{2R}{\epsilon}\)^d
\end{align}
such that
\begin{align}
\forall x \in R\mathcal B, \exists y \in \mathcal D_\epsilon \quad \text{such that} \quad  \| x - y \|_2 \leq \epsilon.
\end{align}
\end{lemma}
\begin{proof}
By scaling the unit ball to have radius $R$ and using lemma 5.2 in  \citep{vershynin2010introduction}.
\end{proof}

\begin{lemma}[Bernstein for Martinglaes]
	\label[lemma]{lem:Bernstein}
	Consider the stochastic process $\{ X_n\}$ adapted to the filtration $\{ \mathcal F_n\}$. Assume $\E X_n = 0$ and $c X_n \leq 1$ for every $n$; then for every constant $z \neq 0$ it holds that
	\begin{align}
		\Pro\(\sum_{n=1}^N X_n \leq z\sum_{n=1}^N\E (X^2_n \mid \mathcal F_n) + \frac{1}{z}\ln\frac{1}{\delta}\) \geq 1-\delta.
	\end{align}
	This implies
	\begin{align}
		\Pro\(\sum_{n=1}^N X_n \leq  c \times \sqrt{\sum_{n=1}^N\E (X^2_n \mid \mathcal F_n)\ln\frac{1}{\delta}} + \ln\frac{1}{\delta}\) \geq 1-\delta.
	\end{align}
\end{lemma}
\begin{proof}
	The first inequality follows from Theorem 1 in \citep{beygelzimer2011contextual}; the second follows from optimizing the bound as a function of $z$ depending upon which term in the rhs is larger.
\end{proof}

\end{document}